\documentclass{article}
\usepackage[nonatbib, final]{neurips_2020}
\usepackage{amsthm,mkolar_definitions,color,comment}
\usepackage{algorithm}

\usepackage{physics}

\usepackage[colorlinks,citecolor=blue,linkcolor=blue]{hyperref}

\usepackage{graphicx}
\usepackage{comment}
\usepackage{color-edits}
\usepackage{xspace}
\usepackage{times}
\usepackage{subfigure}
\usepackage{wrapfig}
\usepackage[font=small]{caption}

\usepackage{algorithm}
\usepackage[noend]{algorithmic}

\usepackage{bm}
\usepackage{bbm}

\usepackage{amsmath, amsthm, amssymb}
\usepackage{mathtools}
\usepackage{multirow}
\usepackage{hhline}
\usepackage{booktabs}
\usepackage{multicol}
\usepackage{xcolor}

\usepackage{enumitem}

\newcommand{\xhdr}[1]{\vspace{1mm} \noindent{\bf #1}}

\newcommand{\corral}{\ensuremath{\mathtt{Corral}}\xspace}

\newcommand{\cats}{\ensuremath{\mathtt{CATS}}\xspace}
\newcommand{\catsop}{\ensuremath{\mathtt{CATS\_Off}}\xspace}
\newcommand{\treetrain}{\ensuremath{\mathtt{Train\_tree}}\xspace}
\newcommand{\tree}{\ensuremath{\Tcal}\xspace}
\newcommand{\dl}{\ensuremath{\mathtt{dLinear}}\xspace}
\newcommand{\dt}{\ensuremath{\mathtt{dTree}}\xspace}

\newcommand{\mina}{a_{\min}\xspace}
\newcommand{\maxa}{a_{\max}\xspace}

\newcommand{\id}{\mathtt{id}\xspace}
\newcommand{\onlyleft}{\mathtt{only\_left}}
\newcommand{\onlyright}{\mathtt{only\_right}}

\newcommand{\rc}{\ensuremath{\mathtt{Return\_cost}}\xspace}
\newcommand{\at}{\alpha}
\newcommand{\bt}{\beta}
\newcommand{\cost}{\mathtt{cost}}
\newcommand{\parent}{\mathtt{parent}}
\newcommand{\learn}{\mathtt{learn}}

\DeclareMathOperator{\Reg}{Reg}

\DeclareMathOperator*{\E}{\mathbb{E}}

\DeclareMathOperator*{\range}{range}
\DeclareMathOperator*{\labelt}{label}

\newcommand{\lt}{\mathtt{l}}
\newcommand{\vt}{\mathtt{v}}
\newcommand{\ut}{\mathtt{u}}
\newcommand{\wt}{\mathtt{w}}
\newcommand{\rt}{\mathtt{r}}
\newcommand{\leftt}{\mathtt{left}}
\newcommand{\rightt}{\mathtt{right}}
\newcommand{\getaction}{\mathtt{get\_action}}
\newcommand{\roott}{\mathtt{root}}
\newcommand{\vol}{\mathrm{vol}}
\newcommand{\onlinetreetrain}{\ensuremath{\mathtt{Online\_train\_tree}}\xspace}

\newcommand{\buildtree}{\ensuremath{\mathtt{Build\_tree}}\xspace}

\newcommand{\defeq}{\triangleq}

\newcommand{\Pen}{\mathrm{Pen}}

\newcommand{\OMIT}[1]{}
\newcommand{\eps}{\ensuremath{\epsilon}}

\usepackage{color-edits}
\addauthor{as}{red}      

\newtheorem*{theorem*}{Theorem}
\newtheorem*{runex*}{Running Example}
\newtheorem*{ex*}{Example}
\newtheorem*{remark*}{Remark}

\usepackage[numbers]{natbib}

\newcount\Comments  
\definecolor{darkgreen}{rgb}{0,0.5,0}
\definecolor{darkred}{rgb}{0.7,0,0}
\definecolor{teal}{rgb}{0.3,0.8,0.8}
\definecolor{orange}{rgb}{1.0,0.5,0.0}
\definecolor{purple}{rgb}{0.8,0.0,0.8}
\newcommand{\kibitz}[2]{\ifnum\Comments=1{\textcolor{#1}{\textsf{\footnotesize #2}}}\fi}

\renewcommand{\citet}{\cite}

\usepackage{units} 

\newcommand{\term}[1]{\ensuremath{\mathtt{#1}}\xspace}

\newcommand{\ExpL}{\lambda} 
\newcommand{\base}{\nu}

\newcommand{\smooth}{\term{Smooth}}

\newcommand{\expf}{\term{EXP4}}

\newcounter{protocol}


\title{Efficient Contextual Bandits with Continuous Actions}

\author{%
Maryam Majzoubi\thanks{\texttt{mm7918@nyu.edu}}\\
New York University \\

\And
Chicheng Zhang\thanks{\texttt{chichengz@cs.arizona.edu}}\\
University of Arizona\\
\And
Rajan Chari\thanks{\texttt{rajan.chari@microsoft.com}}\\
Microsoft Research\\
\And
Akshay Krishnamurthy\thanks{\texttt{akshaykr@microsoft.com}}\\
Microsoft Research\\
\And
John Langford\thanks{\texttt{jcl@microsoft.com}}\\
Microsoft Research\\
\And
Aleksandrs Slivkins\thanks{\texttt{slivkins@microsoft.com}}\\
Microsoft Research\\
}

\begin{document}


\maketitle
\setcounter{footnote}{0}
\begin{abstract}
We create a computationally tractable algorithm for contextual bandits with continuous actions having unknown structure.  Our reduction-style algorithm composes with most supervised learning representations.  We prove that it works in a general sense and verify the new functionality with large-scale experiments.
\end{abstract}

\section{Introduction}

In contextual bandit learning~\citep{auer2002nonstochastic,abe2003reinforcement,Langford-nips07,monster-icml14}, an agent repeatedly observes its environment, chooses an action, and receives a reward feedback, with the goal of optimizing cumulative reward. When the action space is discrete, there are many solutions to contextual bandit learning with successful deployments in personalized health, content recommendation, and elsewhere~\citep[e.g.,][]{li2010contextual,tewari2017ads,DS-arxiv, reb1, reb2, reb3}.  However, in many practical settings the action chosen is actually continuous.  How then can we efficiently choose the best action given the context?  This question is also extremely relevant to reinforcement learning more generally since contextual bandit learning is one-step reinforcement learning.

There are many concrete examples of reinforcement learning problems with continuous actions.  In precision medicine~\cite{chen2016personalized,kallus2018policy}, doctors may prescribe to a patient a medication with a continuous value of dosage~\cite{klein2009estimation}. In data center optimization, the fan speeds and liquid coolant flow may be controllable continuous values~\cite{lazic2018data}.  In operating systems, when a computer makes a connection over the network, we may be able to adjust its packet send rate in response to the current network status~\cite{jay2019deep}.  All of these may be optimizable based on feedback and context.

A natural baseline approach here is to posit smoothness assumptions on the world, as in much prior work, e.g., \citep{agrawal1995continuum, LipschitzMAB-stoc08, bubeck2011x, slivkins2014contextual, cesa2017algorithmic}. This approach comes with practical drawbacks. Many applications do not exhibit any smoothness structure. When/if they do, the smoothness parameters (such as a Lipschitz constant) must be known in advance. Unfortunately, discovering the smoothness parameters is challenging, and requires knowing some other parameters and/or extensive exploration. 





A recent approach to continuous actions~\cite{Krish2019colt} realizes similar performance guarantees without knowing the Lipschitz constant
(let alone a more refined smoothness structure),  while leveraging any preferred policy representation. Here, each action is ``smoothed'' to a distribution over an interval, and the benchmark one competes with is ``smoothed'' similarly. Unfortunately, their algorithm is computationally infeasible since it requires enumeration of all possible policy parameter settings.

In this paper, we realize benefits similar to this approach with a computationally practical algorithm, 
for contextual bandits with continuous action space $[0,1]$.
Our algorithms are {\em oracle-efficient}~\citep{Langford-nips07,dudik2011efficient,monster-icml14,rakhlin2016bistro,syrgkanis2016efficient}: computationally efficient whenever we can
solve certain supervised learning problems.
Our main algorithm chooses actions by navigating a tree with supervised learners acting as routing functions in each node. Each leaf corresponds to an action, which is then ``smoothed'' to a distribution from which the final action is sampled. 
We use the reward feedback to update the supervised learners in the nodes to improve the ``tree policy.'' 

Our contributions can be summarized as follows:

\begin{itemize}[leftmargin=*]
\setlength{\itemsep}{0pt}
\setlength{\parsep}{0pt}
    \item We propose \cats, a new algorithm for contextual bandits with continuous actions (Algorithm~\ref{alg:greedy-tree}). It uses $\eps$-greedy exploration with tree policy classes (Definition~\ref{def:tree-policy-class}) and is implemented in a fully online and oracle-efficient manner. We prove that \cats has prediction and update times scaling as log of the tree size, an exponential improvement over traditional approaches.
        Assuming realizability, \cats has a sublinear regret guarantee against the tree policy class (Theorem~\ref{thm:greedy-tree}).
    \item We propose \catsop, an off-policy optimization version of \cats (Algorithm~\ref{alg:off-policy-opt}) that can utilize logged data to train and select tree policies of different complexities. We also establish statistical guarantees for this algorithm (Theorem~\ref{thm:off-policy-opt}).
    \item We implement our algorithms in Vowpal Wabbit
    (\emph{vowpalwabbit.org}),
    and compare with baselines on real datasets. Experiments demonstrate the efficacy and efficiency of our approach (Section~\ref{sec:expt}).
\end{itemize}

\xhdr{Discussion.} The smoothing approach has several appealing properties. 
We look for a good interval of actions, which is possible even when the best single action is impossible to find. 
We need to guess a good \emph{width}, but the algorithm adjusts to the best location for the interval. 
This is less guessing compared to uniform discretization (where the width and location are tied to some extent). 
While the bandwidth controls statistical performance, an algorithm is free to discretize actions for the sake of computational feasibility. 
An algorithm can improve accuracy by reusing datapoints for overlapping bands. Finally, the approach is principled, leading to specific, easily interpretable guarantees.

The tree-based classifier is a successful approach for supervised learning with a very large number of actions (which we need for computational feasibility). However, adapting it for smoothing runs into some challenges. First, a naive implementation leads to a prohibitively large per-round running time; we obtain an exponential improvement as detailed in Section~\ref{sec:greedy-tree}.
Second, existing statistical guarantees do not carry over to regret in bandits: they merely ``transfer" errors from tree nodes to the root~\cite{beygelzimer2009error,beygelzimer2009offset}, but the former errors could be huge. 
We posit a realizability assumption; even then, the analysis is non-trivial because the errors accumulate as we move down the tree. 


Another key advantage of our approach is that it allows us to use off-policy model selection. 
For off-policy evaluation, we use smoothing to induce exploration distribution supported on the entire action space. 
Hence, we can discover when refinements in tree depth or smoothing parameters result in superior performance. Such model selection is not possible when using discretization approaches.
When employed in an offline setup with data collected by a baseline \emph{logging policy}, our experiments show that off-policy optimization can yield dramatic performance improvements. 

\xhdr{Related work.}
Contextual bandits are quite well-understood for small, discrete action spaces, with rich theoretical results and successful deployments in practice. To handle large or infinite action spaces, most prior work either makes strong parametric assumptions such as linearity, or posits some continuity assumptions such as Lipschitzness. More background can be found in 
~\citet{bubeck2012regret,slivkins2019introduction,lattimore2018bandit}.

Bandits with Lipschitz assumptions were introduced
in~\citet{agrawal1995continuum}, and optimally solved in the worst
case by
\citet{Bobby-nips04}. \citet{LipschitzMAB-stoc08,kleinberg2013bandits,bubeck2011x,slivkins2014contextual}
achieve optimal data-dependent regret bounds, while several papers
relax global smoothness assumptions with various local
definitions~\citep{auer2007improved,LipschitzMAB-stoc08,kleinberg2013bandits,bubeck2011x,ImplicitMAB-nips11,minsker2013estimation,grill2015black}.
This literature mainly focuses on the non-contextual version, except for
\citet{slivkins2014contextual,NIPS2011_4487,cesa2017algorithmic,wang2019towards} (which only consider a fixed policy set $\Pi$). As argued in \citet{Krish2019colt}, the smoothing-based approach is productive in these settings, and extends far beyond, e.g., to instances when the global optimum is a discontinuity.

Most related to this paper is \citep{Krish2019colt}, which introduces the smoothness approach to contextual bandits and achieves data-dependent and bandwidth-adaptive regret bounds. Their approach extends to generic ``smoothing distributions" (kernels), as well as to adversarial losses. However, their algorithms are inherently computationally inefficient, because they build on the techniques from \citep{auer2002nonstochastic,dudik2011efficient}.

Our smoothing-based reward estimator was used in~\citet{Krish2019colt} for contextual bandits, as well as in~\citet{kallus2018policy, chen2016personalized} in the observational setting. The works of~\citet{kallus2018policy, chen2016personalized} learn policies that are linear functions of the context, and perform policy optimization via gradient descent on the IPS loss estimate.





\section{Preliminaries}
\label{sec:prelims}

\noindent {\bf Setting and key definitions.}
We consider the stochastic (i.i.d.) contextual bandits (CB) setting.  At each round $t$, the environment produces a (context, loss) pair $(x_t,\ell_t)$ from a distribution $\Dcal$.
Here, context $x_t$ is from the context space $\Xcal$, and the loss function $\ell_t$ is a mapping from the action space $\Acal \defeq [0,1]$ to $[0,1]$.
Then, $x_t$ is revealed to the learner, based on which it chooses an action $a_t \in \Acal$ and observes loss $\ell_t(a_t)$. The learner's goal is to minimize its cumulative loss,
$\sum_{t=1}^T\ell_t(a_t)$.

Define a \emph{smoothing operator}:
 $\smooth_h : \Acal \to \Delta(\Acal)$, that maps each action $a$ to a uniform distribution over the interval $\{a'\in \Acal:\; |a - a'| \leq h\} = [a-h,a+h]\cap [0,1]$.
 As notation, let $\base$ denote the Lebesgue measure, i.e. the uniform distribution over $[0,1]$. Denote by
$\smooth_h(a'|a)$ the probability density function w.r.t., $\base$ for $\smooth_h(a)$ at action $a'$.
We define $\smooth_0(a) \defeq \delta_a$, where $\delta_a$ is the Dirac point mass at $a$.
For a \emph{policy} $\pi: \Xcal \to \Acal$, we define $\pi_h(a'|x) \defeq \smooth_h(a'|\pi(x))$ to be the probability density value for action $a'$ of the smoothed policy $\pi_h$ on context $x$.

Equivalently, we define $h$-smoothed loss $\ell_h(a) \defeq \EE_{a' \sim \smooth_h(\cdot | a)} \sbr{\ell(a')}$. For policy $\pi: \Xcal \to \Acal$ we define the corresponding \emph{$h$-smoothed expected loss} as $\ExpL_h(\pi) \defeq \EE_{(x,\ell)\sim \Dcal}\; \EE_{a \sim \smooth_h(\pi(x))} \sbr{ \ell(a) }$. This is equivalent to defining $\pi_h: x \mapsto \smooth_h(\pi(x))$, and evaluating $\pi_h$ on the original loss, i.e., $\lambda_0(\pi_h) = \lambda_h(\pi)$.
The bandwidth $h$ governs an essential bias-variance trade-off in the continuous-action setting: with small $h$, the smoothed loss $\lambda_h(\pi)$ closely approximates the true expected loss function $\lambda_0(\pi)$,
whereas the optimal performance guarantees scale inversely with $h$.



Over the $T$ rounds, the learner accumulates a history of interaction. After round $t$, this is $(x_s, a_s, p_s, \ell_s(a_s))_{s=1}^{t}$, where $x_s$ is the context, $a_s$ is the chosen action, $p_s = P_s(a_s \mid x_s)$ is the value of the density $P_s(\cdot \mid x_s)$ used at round $s$ at $a_s$, and $\ell_s(a_s)$ is the observed loss. From this history, we use an {\em inverse propensity score} (IPS) estimator~\cite{horvitz1952generalization} to compute an unbiased estimate of the smoothed loss $\lambda_h(\pi)$: $\hat{V}_{t}(\pi_h) = \frac1t\sum_{s=1}^t \frac{\pi_h(a_s \mid x_s)}{P_s(a_s \mid x_s)} \ell_s(a_s)$. A useful primitive in contextual bandits is to find a policy $\pi$ that minimizes $\hat{V}_t(\pi_h)$, which is a surrogate for $\lambda_0(\pi)$.


A natural approach for policy optimization is to reduce to cost-sensitive multiclass classification (CSMC).
We choose a discretization parameter $K$, and instantiate a policy class $\Pi: \Xcal \to \Acal_K$
where $\Acal_K \defeq \{0,\tfrac{1}{K}, \tfrac{2}{K},\ldots,\tfrac{K-1}{K}\}$.
Then, as
$ \pi_h(a_s \mid x_s) = \smooth_h( a_s \mid \nicefrac i K ), \text{ if } \pi(x_s) = \nicefrac i K $,
policy optimization can be naturally phrased as a CSMC problem.
For each round, we create a cost-sensitive example $(x_s, \tilde{c}_s)$ where $\tilde{c}_s(\nicefrac i K) = \frac{\ell_s(a_s)}{P_s(a_s \mid x_s)} \smooth_h \rbr{a_s \lvert \nicefrac i K}$, for all $\nicefrac i K$ in $\Acal_K$.
Then, optimizing $\hat{V}_t(\pi_h)$ is equivalent to computing $\argmin_{\pi \in \Pi} \sum_{s=1}^t \tilde{c}_s(\pi(x_s))$.
When working with $h$-smoothed losses, the error incurred by using the discretized action space $\Acal_K$ can be controlled, as we can show that $\ell_h(a)$ is $\nicefrac{1}{h}$-Lipschitz~\citep{Krish2019colt}.\footnote{Although we use the $1/h$-Lipschitz property of $h$-smoothed losses here, in general, $h$-smoothed losses have more structure than $1/h$-Lipschitz losses, which admit better regret guarantees in general.} So, this discretization strategy can compete with policies that are not restricted to $\Acal_K$, incurring an additional error of $\tfrac{1}{hK}$ per round.

\xhdr{Tree policies.}
One challenge with applying the CSMC approach is computational: for general classes $\Pi$, classical methods for CSMC (such as one-versus-all)
have $\Omega(K)$ running time.
This is particularly problematic since we want $K$ to be quite large in order to compete with policies that are not restricted to $\Acal_K$.
To overcome this challenge, we consider a structured policy class induced by a binary tree $\Tcal$, where each node $\vt$ is associated with a binary classifier $f^\vt$ from some base class $\Fcal$.\footnote{We assume that $\Fcal$ is finite for simplicity. This can be extended with empirical process arguments.}

\begin{definition}[Tree policy]
Let $K = 2^D$ for some natural number $D$, and $\Fcal$ be a class of binary classifiers from $\Xcal$ to $\cbr{\leftt, \rightt}$. $\tree$ is said to be a tree policy over action space $\Acal_K = \cbr{\nicefrac{i}{K}}_{i=0}^{K-1}$ using $\Fcal$, if: (1) $\tree$ is a complete binary tree of depth $D$ with $K = 2^D$ leaves, where each leaf $\vt$ has label $\labelt(\vt) = \nicefrac{0}{K}, \ldots, \nicefrac{K-1}{K}$ from left to right, respectively; (2) in each internal node $\vt$ of $\tree$, there is a classifier $f^{\vt}$ in $\Fcal$; (3) the prediction of $\tree$ on an example $x$, $\tree.\getaction(x)$, is defined as follows. Starting from the root of $\tree$, repeatedly route $x$ downward by entering the subtree that follows the prediction of the classifier in the tree nodes. When a leaf is reached, its label is returned (see Algorithm~\ref{alg:filter-tree-pred} in Appendix~\ref{sec:add-not} for a formal description).
\end{definition}


In other words, a tree policy over action space $\Acal_K$ can be viewed as a decision tree of depth $D$, where its nodes form a hierarchical partition of the discretized action space $\Acal_K$. For each node in the tree, there is a subset of the context space that gets routed to it; therefore, given a tree policy $\tree$ over $\Acal_K$, it also implicitly defines a hierarchical partition of the context space $\Xcal$.
The crucial difference between a tree policy and a decision tree in the usual sense, is that each leaf node corresponds to a distinct action.
Our tree policy approach is also fundamentally different from the approach of~\cite{slivkins2014contextual,wang2019towards} in contextual bandits, in that their usages of trees are in performing regression of reward as a function of (context, action) pairs.
Our policy classes of interest are tree policy classes:
\begin{definition}
\label{def:tree-policy-class}
Let $\Fcal_{K}$ denote the policy class of all tree policies over action space $\Acal_{K}$ using base class $\Fcal$, that is, the set of tree policies $\Fcal_{K} = \cbr{\tree: \tree \text{ is a tree policy over action space $\Acal_K$ using $\Fcal$} }$.
Furthermore, Let $\Fcal_\infty$ denote the policy class of all tree policies of arbitrary depths using base class $\Fcal$, formally, $\Fcal_\infty = \bigcup_{K: K \in 2^{\NN}}\Fcal_K$.
\end{definition}

As a computational primitive, we assume that we can solve CSMC problems over the base class $\Fcal$. Note that formally these are \emph{binary} classification problems.
The main advantage of using these structured policy classes is computational efficiency. As we demonstrate in the next section, we can use fast online CSMC algorithms to achieve a running time of $\order{\log K}$ per example. At the same time, due to the hierarchical structure, choosing an action using a policy in $\Fcal_K$ also takes $\order{\log K}$ time. Both of these are exponentially faster than the $\order{K}$ running time that typically arises from flat representations. Finally, given a tree policy, we define the tree policy rooted at one of its nodes:

\begin{definition}
Let $\vt$ be an internal node in $\tree$. We define $\Tcal^\vt$ as the tree-based policy with root at $\vt$. We will abbreviate $\tree^\vt.\getaction(x)$ as $\tree^\vt(x)$ or $\vt.\getaction(x)$.
\end{definition}



\vspace{-2mm}

\xhdr{The performance benchmark.}
We define the performance benchmark:
$
    \Reg(T,\Fcal_\infty,h) \defeq \EE\sbr{\sum_{t=1}^T \ell_t(a_t)} - T\inf_{\pi \in \Fcal_{\infty}} \lambda_h(\pi).
$
In words, we are comparing the cumulative expected loss of our algorithm, with the $h$-\emph{smoothed} cumulative expected loss of the best tree policy of arbitrary depth. We call this the $h$-\emph{smoothed regret} w.r.t. $\Fcal_{\infty}$. Although the focus of this paper is on contextual bandit algorithms with computational efficiency guarantees, in Appendix~\ref{sec:general}, we also present several extensions of our results to general policy classes.

\xhdr{Miscellaneous notation.}
Given a set of CSMC examples $S = \cbr{(x,c)}$ of size $n$, and a function $f$, we use $\EE_S \sbr{f(x,c)} \defeq \frac1n\sum_{(x,c) \in S} f(x,c)$ to denote empirical expectation of $f$ over $S$.
Given a function $f$ with domain $\Zcal$, define its range to be the set of values it can take, i.e. $\range(f) \defeq \cbr{f(z): z \in \Zcal}$. Specifically, given a tree $\tree$ over action space $\Acal_K$ and a node $\vt$
 in $\tree$, $\range(\tree^\vt)$ denotes the actions reachable by $\tree$, i.e. the action labels of the leaves that are descendants of $\vt$. Given a natural number $n$, we denote by $[n] \defeq \cbr{1,\ldots,n}$.

\section{Algorithm}
\label{sec:algorithm}

We describe our main algorithm \cats for learning with continuous actions using tree policies in Algorithm~\ref{alg:greedy-tree} and an off-policy version in Algorithm~\ref{alg:off-policy-opt} for unknown $h$. In Appendix~\ref{sec:adaptive-tree}, we also present a variant of \cats that works online for unknown $h$.

\subsection{Smoothed $\eps$-greedy algorithm with trees}
\label{sec:greedy-tree}

We present Algorithm~\ref{alg:greedy-tree} in this section. It consists of two main components: first, a smoothed $\epsilon$-greedy exploration strategy (lines~\ref{line:greedy-start} to~\ref{line:greedy-end}); second,
a tree training procedure \treetrain called at line~\ref{line:tree-train}, namely Algorithm~\ref{alg:tree-train}. We discuss each component in detail next.

	
\xhdr{$\epsilon$-greedy exploration with smoothing.} At time step $t$, the algorithm uses the policy $\pi_t$ learned from data collected in previous time steps to perform action selection. Specifically, with probability $\epsilon$, it chooses an action uniformly at random from $\Acal$; otherwise, it chooses an action based on the prediction of $\pi_{t,h}$, the $h$-smoothing of policy $\pi_t$. As we will see, $\pi_{t,h}$ has expected loss competitive with any smoothed policy $\pi_h$ with $\pi$ in $\Fcal_K$ (and is therefore competitive with $\Fcal_\infty$).
This component is similar to the $\epsilon$-greedy algorithm for discrete action contextual bandits~\citep[e.g.][]{Langford-nips07}; here $\epsilon$ is a parameter that trades off between exploration and exploitation, where a larger $\epsilon$ yields better quality data for learning, and a smaller $\epsilon$ implies actions with better instantaneous losses are taken.


\xhdr{Tree training.} Given the interaction log collected up to time $t$, $\cbr{(x_s, a_s, P_s(a_s \mid x_s), \ell_t(a_s))}_{s=1}^t$, Algorithm~\ref{alg:greedy-tree} incorporates it to produce a policy $\pi_{t+1}$ for time $t+1$. Specifically, $\pi_{t+1}$ is a tree policy $\tree$ in $\Fcal_K$ that approximately minimizes $\hat{V}_t(\tree'_h)$ over all policies $\tree'$ in $\Fcal_\infty$. To this end,
we use \treetrain (Algorithm~\ref{alg:tree-train}) over the set of cost-sensitive examples $(x_s, \tilde{c}_s)_{s=1}^t$ constructed by IPS.
For technical reasons\footnote{We need to partition the input CSMC dataset in a delicate manner to ensure \treetrain's theoretical guarantees; see Lemma~\ref{lem:off-policy-fixed} and its proof in Appendix~\ref{sec:proofs} for more details. In our implementation we ignore such subtlety; see Algorithm~\ref{alg:tree-train-no-part} in Appendix~\ref{sec:cats-implementation}.}, \treetrain differs from the filter tree algorithm~\cite{beygelzimer2009error} in that it partitions the dataset into $D = \log K$ subsets, with their indices $B_0, \ldots, B_{D-1}$ being disjoint subsets in $[n]$. For every $i$, the examples with indices in $B_i$ are dedicated to training classifiers in tree nodes at level $i$.
\treetrain trains the classifiers in the tree nodes in a bottom-up fashion. At the bottom layer, each node $\vt$ with two leaves as children seeks a classifier $f_\vt$ in $\Fcal$ that directly classifies the context $x$ to the action in $\range(\tree^\vt)$ with smaller expected cost $\EE[\ell(a) \mid x]$.
For this, it invokes CSMC learning with class $\Fcal$ where costs are the IPS costs for the two children $\vt.\leftt$ and $\vt.\rightt$.
At other internal nodes $\vt$, given that all the downstream classifiers in subtrees rooted at $\vt.\leftt$ and $\vt.\rightt$ have been trained, it aims to find a classifier $f_\vt$ in $\Fcal$ such that $f_\vt$, in conjunction with other classifiers in $\tree^\vt$, routes context $x$ to the action in $\range(\tree^\vt)$ with the smallest expected cost. 

\begin{algorithm}[t]
\caption{\cats: continuous action tree with smoothing}
\label{alg:greedy-tree}
	\begin{algorithmic}[1]
    \REQUIRE{Exploration and smoothing parameters $(\epsilon, h)$, discretization scale $K = 2^D$, base class $\Fcal$}
    \STATE Initialize dataset $S_0 \gets \emptyset$, and tree policy $\tree$ with classifiers $f^\vt \in \Fcal$ at every internal node $\vt$.
    \label{line:cats-init}
	\FOR{$t = 1,2,\ldots, T$}	
	\STATE Let $\pi_t$ be the tree policy $\tree$ with $\cbr{f^\vt}$ as classifiers.
	\label{line:tree-t}
    \STATE Define policy $P_t(a \mid x) := (1-\epsilon) \pi_{t,h}(a|x) + \epsilon$.
    \label{line:policy-t}
    \label{line:greedy-start}
    \STATE Observe context $x_t$, select action $a_t \sim P_t(\cdot \mid x_t)$, observe loss $\ell_t(a_t)$.
    \label{line:cb-inter}
    \STATE 
    Let
    $\tilde{c}_t(\nicefrac{i}{K}) \gets \frac{\smooth_h(a_t \mid \nicefrac{i}{K})}{P_t(a_t \mid x_t)} \ell_t(a_t)$ for all $i$ 
    \label{line:greedy-end}
    \label{line:csmc-gen}
    \STATE $\Tcal \gets \treetrain(K,\Fcal,\{(x_s,\tilde{c}_s)\}_{s=1}^t)$.
    \label{line:tree-train}
	\ENDFOR
\end{algorithmic}
\end{algorithm}
\begin{algorithm}[t]
 \caption{Tree training: \treetrain}\label{alg:tree-train}
 \begin{algorithmic}[1]
    \REQUIRE $K=2^D$, $\Fcal$, data $\{(x_s,c_s)\}_{s=1}^n$ with $c_s \in \RR^K$
    \STATE For level $d = 0,\ldots,D-1$: $B_{d} \gets \cbr{ (D-d-1)n' + 1, \ldots, (D-d) n' }$, where $n' = \lfloor \nicefrac{n}{D} \rfloor $.
    \FOR{level $d$ from $D-1$ down to $0$}
    \FOR{nodes $\vt$ at level $d$}  
    \STATE For each $(x_s,c_s)$ define binary cost $c_s^{\vt}$ with\\ {$c_s^{\vt}(\leftt) = c_s(\vt.\leftt.\getaction(x_s)), c_s^{\vt}(\rightt) = c_s(\vt.\rightt.\getaction(x_s)).$}
    \STATE Train classifier at node $\vt$: $f^{\vt} \gets \argmin_{f \in \Fcal} \EE_{S^\vt} \sbr{c^\vt(f(x))}$, where $S^{\vt} = \{(x_s,c_s^{\vt}): s \in B_l, c_s^{\vt}(\leftt) \ne c_s^{\vt}(\rightt)\}$:
    \label{line:tree-node-erm}
    \ENDFOR
    \ENDFOR
    \RETURN tree $\Tcal$ with $\{f^{\vt}\}$ as node classifiers.
 \end{algorithmic}
\end{algorithm}
\begin{algorithm}[t]
\caption{\catsop}
\begin{algorithmic}[1]
\REQUIRE{logged data $\cbr{ (x_t, a_t, P_t(a_t \mid x_t), \ell_t(a_t)) }_{t=1}^T$, minimum density of action distribution $p_{\min}$, set of (bandwidth, discretization level) combinations $\Jcal \subset [0,1] \times 2^\NN$, base class $\Fcal$.}
\FOR{(bandwidth, discretization level) $(h,K)$ in $\Jcal$}
\label{line:po-tree-start-off}
	\STATE For every $t$ in $[T]$, let
    $\tilde{c}_t^h(\nicefrac{i}{K}) \gets \frac{\smooth_h(a_t \mid \nicefrac{i}{K})}{P_t(a_t \mid x_t)} \ell_t(a_t)$ for all $i \in \cbr{0,\ldots,K-1}$.
	\FOR{$t = 1,2,\ldots,T$}
	\STATE $\tree^{h, K}_t \gets \treetrain(K,\Fcal,\cbr{(x_s, \tilde{c}_s^h)}_{s=1}^{t-1})$. 	
	\ENDFOR
\ENDFOR
\label{line:po-tree-end-off}
	\STATE Let
$
	(\hat{h}, \hat{K}) \gets \argmin_{(h,K) \in \Jcal} \Big( \frac1T \sum_{t=1}^T \tilde{c}_t^h(\tree^{h, K}_t(x_t))
	+ \Pen(h, K) \Big),
$
	where $\Pen(h, K) =
	\sqrt{ \frac1T \sum_{t=1}^T \tilde{c}_t^h(\tree^{h, K}_t(x_t)) \cdot \frac{64 \ln\frac {4 T \abr{\Jcal}} \delta}{T p_{\min} h}} + \frac{64 \ln\frac {4 T \abr{\Jcal}} \delta}{T p_{\min} h}.$
	\label{line:srm-off}
	\RETURN $\hat{\pi}$ drawn uniformly at random over set $\cbr{\tree^{\hat{h}, \hat{K}}_{t, \hat{h}}}_{t=1}^T$.
\end{algorithmic}
\label{alg:off-policy-opt}
\end{algorithm}

\xhdr{Computational complexity.} \cats can be implemented in a fully online and oracle-efficient fashion, using online CSMC learners.
Specifically, line~\ref{line:tree-node-erm} in \treetrain can be implemented by maintaining a stateful online learner for each tree node $\vt$, which at time $t$ maintains $f_t^{\vt}$, an approximation of $\argmin_{f \in \Fcal} \sum_{s=1}^{t-1} \sbr{c_s^\vt(f(x_s))}$. Then, upon seeing a binary CSMC example $(x_t, c_t^\vt)$, the learner employs incremental update rules such as stochastic gradient descent to update its internal state to $f_{t+1}^{\vt}$, an approximate solution to the next CSMC problem.




We now look at the per-example computational cost of \cats using the above online implementation of CSMC oracle.
Naively, in line~\ref{line:tree-node-erm} of \treetrain, if we instead define $S^\vt = \cbr{(x_s, c^\vt_s): s \in B_i}$ for every node $\vt$, i.e. we do not filter out examples with identical costs for left and right sides at node $\vt$, the time for processing each example would be $\order{ K }$, since it contributes a binary CSMC example to $S^\vt$ for every node $\vt$.

Our first observation is that, if at time $t$, $c_t^{\vt}(\leftt) = c_t^{\vt}(\rightt)$, the online CSMC learner can skip processing example $(x_t, c_t^\vt)$, as is done in line~\ref{line:tree-node-erm} of \treetrain. This is because adding this example does not change the cost-sensitive ERM from round $t$ to round $t+1$.
However, the algorithm still must decide whether this happens for each node $\vt$, which still requires $\order{K}$ time naively.

Our second observation is that, by carefully utilizing the piecewise constant nature of the IPS cost vector $\tilde{c}_t$, we can find the nodes that need to be updated and compute the costs of their left and right children, both in $\order{\log K}$ time per example. Specifically, as $\tilde{c}_t$ is piecewise constant with two discontinuities, only two root-to-leaf paths contain nodes that have children with differing costs and must be updated (see Appendix~\ref{sec:cats-implementation}, specifically Lemma~\ref{lem:update-correctness} and its proof for more explanations).
Exploiting these observations, we implement \cats to have $\order{\log K}$ update time, an exponential improvement over naive implementations. This is summarized in the next theorem.

\begin{theorem}\label{thm:cats-log-time}
\cats with an online learner at each node requires $\order{\log K}$ computation per example.
\end{theorem}

We elaborate on our online implementation and present the proof of the theorem in Appendix~\ref{sec:cats-implementation}. Our theorem generalizes the computational time analysis of the offset tree algorithm for discrete-action contextual bandits~\cite{beygelzimer2009offset}, in that we allow input IPS CSMC examples to have multiple nonzero entries.

\subsection{Off-policy optimization}
As discussed above, one major advantage of the smoothing approach to  contextual bandits with continuous actions is that policy optimization can be easily reduced to
a CSMC learning problem via counterfactual techniques. This allows off-policy optimization, in the sense that the logged data can be collected using one policy that takes action in $\Acal$, while we can optimize over (smoothed) policy classes that take actions in $\Acal_{K}$. In the special setting that we learn from a tree policy class $\Fcal_{K}$, the \catsop algorithm (Algorithm~\ref{alg:off-policy-opt}) can be used.


The algorithm receives an interaction log $\cbr{(x_t, a_t, P_t(a_t \mid x_t), \ell_t(a_t)}_{t=1}^T$, collected by another algorithm such that $P_t( a \mid x_t) \geq p_{\min}$ for all $a \in \Acal$,
a collection of (bandwidth, disretization levels) $\Jcal$, and a base policy class $\Fcal$ as input.
It consists of two stages: tree training and policy selection.
In the tree training stage (lines~\ref{line:po-tree-start-off} to~\ref{line:po-tree-end-off}), for each $( (h, K), t )$ combination
in $\Jcal \times [T]$, the algorithm again calls \treetrain over cost-sensitive examples $\cbr{(x_s, \tilde{c}_s^h)}_{s=1}^{t-1}$ induced by the interaction log and the bandwidth $h$. As a result, we obtain a set of tree policies $\cbr{\tree^{h, K}_t: t \in [T], (h,K) \in \Jcal}$
In the policy selection stage (line~\ref{line:srm-off}), we choose a pair $(\hat{h}, \hat{K})$ from the set $\Jcal$ using structural risk minimization~\cite{vapnik1995nature}, by trading off $\frac1T \sum_{t=1}^T \tilde{c}_t^h(\tree^{h, K}_t(x_t))$, the progressive validation loss estimate of smoothed policies $\cbr{\tree^{h, K}_{t,h}: t \in [T]}$ on logged data~\cite{PV99} and its deviation bound $\Pen(h, K)$ that depends on $h$ and $K$.
A similar procedure has been proposed in the discrete-action contextual bandit learning setting~\cite{swaminathan2015counterfactual}.
As we see from Theorem~\ref{thm:off-policy-opt} below, the obtained tree policy $\tree$ has expected loss competitive with all policies in the set $\cup_{(h,K) \in \Jcal} \cbr{\pi_h: \pi \in \Fcal_K}$.




\section{Performance guarantees}
\label{sec:theory}

In this section, we show that \cats and \catsop achieve sublinear regret or excess loss guarantees
under a realizability assumption over the (context, loss) distribution
$\Dcal$. We defer the formal statements of our theorems and their proofs to Appendix~\ref{sec:proofs}.

As learning decision trees is computationally hard in general~\cite{hancock1996lower}, many existing positive results pose strong assumptions on the learning model, such as uniform or product unlabeled distribution~\cite{blanc2019top,brutzkus2019optimality}, separability~\cite{ehrenfeucht1989learning,rivest1987learning,blum1992rank} or allowing membership queries~\cite{kushilevitz1993learning,gopalan2008agnostically}. Our tree policy training guarantee under the following realizability assumption is complementary to these works:
\begin{definition}
A hypothesis class $\Fcal$ and data distribution $\Dcal$ is said to be $(h, K)$-realizable, if there exists a tree policy $\tree$ in $\Fcal_K$ such that the following holds:
for every internal node $\vt$ in $\tree$, there exists a classifier $f^{\vt,\star}$ in $\Fcal$, such that
\begin{align*}
&f^{\vt,\star}(x) = \leftt
\Rightarrow \min_{a \in \range(\tree^\lt)} \EE[\ell_h(a) \mid x] \leq \min_{a \in \range(\tree^\rt)} \EE[\ell_h(a) \mid x],\\
 &f^{\vt,\star}(x) = \rightt
 \Rightarrow \min_{a \in \range(\tree^\rt)} \EE[\ell_h(a) \mid x] \leq \min_{a \in \range(\tree^\lt)} \EE[\ell_h(a) \mid x],
\end{align*}
where $\lt = \vt.\leftt$ and $\rt = \vt.\rightt$ are $\vt$'s two children; recall that $\ell_h(a) \defeq \EE_{a' \sim \smooth_h(\cdot|a)} \ell(a')$.
\label{def:well-specified-lh}
\end{definition}


Intuitively, the above realizability assumption states that our base class $\Fcal$ is expressive enough, such that for every discretization parameter $K$ in $2^\NN$, there exists a set of $(K-1)$ classifiers $\cbr{f^{\vt,\star}}_{\vt \in \tree} \subset \Fcal$ occupying the internal nodes of a tree $\tree$ of $K$ leaves, and $\tree$ routes any context $x$ to its Bayes optimal discretized action in $\Acal_K$, formally $\argmin_{a \in \Acal_K} \EE[\ell_h(a) \mid x]$.
As $\ell_h(\cdot)$ is $\frac1h$-Lipschitz, $\min_{a \in \Acal_K}\EE[\ell_h(a) \mid x] - \min_{a \in \Acal}\EE[\ell_h(a) \mid x] \leq \frac{1}{hK}$.
This implies that, if $K$ is large enough,
the Bayes optimal policy $\pi^\star(x) = \argmin_{a \in \Acal} \EE[\ell_h(a) \mid x]$ can be well-approximated by a tree policy in $\Fcal_K$ with little excess loss. Under the above realizability assumption, we now present a theorem that characterizes the regret guarantee when Algorithm~\ref{alg:greedy-tree} uses policy class $\Fcal_K$.





\begin{theorem}[Informal]\label{thm:greedy-tree}
Given $h$, suppose $(\Fcal, \Dcal)$ is $(h, K)$-realizable for any $K \in 2^{\NN}$.  Then with appropriate settings of greedy parameter $\epsilon$ and discretization scale $K$, with high probability, Algorithm~\ref{alg:greedy-tree} run with inputs $\epsilon, h, K, \Fcal$ has regret bounded as:
$\Reg(T, \Fcal_\infty, h)
  \leq
  \order{
  	  \rbr{
	  	  \nicefrac{T^4 \ln\abr{\Fcal}}{h^3}
		  }^{\nicefrac{1}{5}}
		}$.
\end{theorem}

We remark that we actually obtain a stronger result: with appropriate
tuning of $\epsilon$, the $h$-smoothed regret of \cats against
$\Fcal_K$ is $\order{ \rbr{\nicefrac{K^2T^2\ln
      \frac{\abr{\Fcal}}{\delta}}{h}}^{1/3}}$, which is similar to the
$\order{T^{\nicefrac{2}{3}}|\Acal|^{\nicefrac{1}{3}}}$ regret for
$\epsilon$-greedy in the discrete actions setting,
with $\nicefrac1h$ serving as the ``effective number of actions.''
We also note that if we used exact ERM over $\Fcal_K$ instead of the computationally efficient \treetrain procedure, the dependence on $K$ would improve from $K^{\nicefrac{2}{3}}$ to $K^{\nicefrac{1}{3}}$.
This $K^{\nicefrac{2}{3}}$ dependence is due to compounding errors accumulating in each node, and we conjecture that it is the price we have to pay for using the computationally-efficient \treetrain for approximate ERM.

The aforementioned $h$-smoothed regret bound against $\Fcal_K$ reflects a natural bias-variance tradeoff in the choice of $h$ and $K$: for a smaller value of $h$, the $h$-smoothed loss more closely approximates the true loss, while achieving a low $h$-smoothed regret bound is harder. A similar reasoning applies to $K$: For larger $K$,  $\Fcal_K$ more closely approximates $\Fcal_\infty$, while the regret of \cats against $\Fcal_K$ can be higher.

We now present learning guarantees of \catsop under the same realizability assumption.




\begin{theorem}[Informal]
Suppose $(\Fcal, \Dcal)$ is $(h, K)$-realizable for all $(h,K) \in \Jcal$. In addition, the logged data $S = \cbr{(x_t, a_t, P_t(a_t \mid x_t), \ell_t(a_t))}_{t=1}^T$ has a sufficient amount of exploration: $P_t(a \mid x_t) \geq p_{\min}$. Then, with high probability, Algorithm~\ref{alg:off-policy-opt} run with inputs $S, p_{\min}, \Jcal, \Fcal$ outputs a policy $\hat{\pi}$ such that: $\lambda_0(\hat{\pi})
	\leq \min_{(h,K) \in \Jcal, \pi \in \Fcal_K}
	             \rbr{ \lambda_h(\pi) +  \order{K \sqrt{\nicefrac{\ln\frac{\abr{\Fcal} \abr{\Jcal}}\delta}{(p_{\min} h T)}} } }$.
	             \label{thm:off-policy-opt}
\end{theorem}

The above theorem shows the adaptivity of \catsop: so long as the logged data is generated by an sufficiently explorative logging policy,
its learned policy is competitive with any policy in the set $\cup_{(h,K) \in \Jcal}\cbr{\pi_h: \pi \in \Fcal_K}$, under realizability assumptions.



\section{Experiments}
\label{sec:expt}
Following the contextual bandit learning evaluation protocol of~\cite{bietti2018contextual}, we evaluate our approach on six large-scale regression datasets, where regression predictions are treated as continuous actions in $\Acal = [0,1]$.
To simulate contextual bandit learning, we first perform scaling and offsetting to ensure $y_t$'s are also in $[0,1]$.
Every regression example $(x_t, y_t)$ is converted to $(x_t, \ell_t)$, where $\ell_t(a) = \abs{a - y_t}$ is the absolute loss induced by $y_t$. When action $a_t$ is taken, the algorithm receives bandit feedback $\ell_t(a_t)$, as opposed to the usual label $y_t$.

Of the six datasets, five are selected from OpenML with the criterion of having millions of samples with unique regression values (See Appendix~\ref{sec:experimental_details} for more details).
We also include a synthetic dataset \texttt{ds}, created by the linear regression model with additive Gaussian noise. 

\xhdr{Online contextual bandit learning using \cats.}
We compare \cats with two baselines that perform $\epsilon$-greedy contextual bandit learning~\cite{Langford-nips07} over the discretized action space $\Acal_K$. The first baseline, \dl, reduces policy training to cost-sensitive one-versus-all multiclass classification~\cite{beygelzimer2005weighted} which takes $\order{K}$ time per example. The second baseline, \dt, uses the filter tree algorithm~\cite{beygelzimer2009error} as a cost-sensitive multiclass learner for policy training, which takes $\order{\log K}$ time per example, but does not perform information sharing among actions through smoothing. 
We run \cats with $(h, K)$ combinations in the following set:
\begin{align}
\Jcal = \cbr{ (h,K): h \in \cbr{2^{-13}, \ldots, 2^{-1}},  K \in \cbr{2^2, \cdots, 2^{13}},
 h K \in \cbr{2^0, \ldots, 2^{11}} }.  \label{eqn:j-expt}
\end{align}
We also run \dl and \dt with values of $K$ in $\cbr{2^1, 2^2, \cdots, 2^{13}}$. All algorithms use $\epsilon = 0.05$; see Appendix~\ref{sec:experimental_details} for additional experimental details.

\begin{figure}[t!]
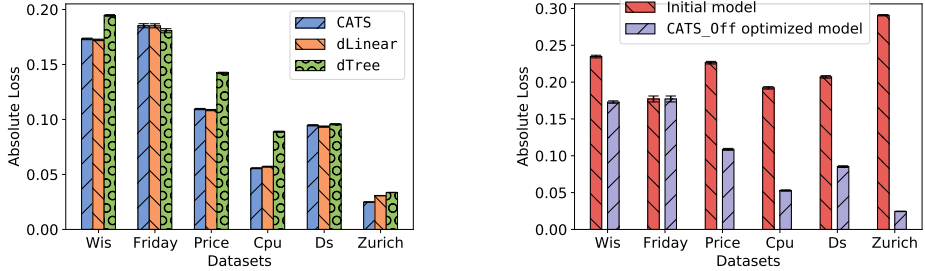

\centering
\includegraphics[width=0.4\textwidth]{Figures/abs_online.png}
\hspace{1cm}
 \centering
\includegraphics[width=0.4\textwidth]{Figures/abs_offline_srm.png}
\caption{({\bf left}) Best progressive validation losses obtained by parameter search for different online learning algorithms on six regression datasets.
({\bf right}) Test-set absolute losses for initial online-trained model using \cats with an initial set of discretization and smoothing parameter $(K_{\texttt{init}}, h_{\texttt{init}}) = (4, 1/4)$, and off-policy optimized models output by \catsop. All confidence intervals are calculated with a single run using the Clopper-Pearson interval with 95\% confidence level (note that they are very small for most of the datasets).}
\label{fig:offline_abs}
\end{figure}





In the left panel of Figure~\ref{fig:offline_abs} we compare \cats with \dl and \dt.
Using progressive validation~\citep{PV99} for online evaluation, our algorithm (with optimally-tuned discretization and bandwidth) achieves performance similar to \dl, and is better than \dt for most of the datasets.


As discussed in Section~\ref{sec:algorithm}, the time cost of our implementation of \cats is $\order{\log(K)}$ per example. Figure~\ref{fig:time} demonstrates that the training time of \cats is constant w.r.t. bandwidth $h$,  and grows logarithmically w.r.t. the discretization $K$. This shows that \cats has the same computational complexity as \dt. In contrast, \dl has $\order{K}$ time complexity per example. The time improvement of \cats compared with \dl becomes more significant when $K$ becomes larger. In summary, \cats outperforms \dt statistically, and has much better scalability than \dl.
\begin{figure}[t!]
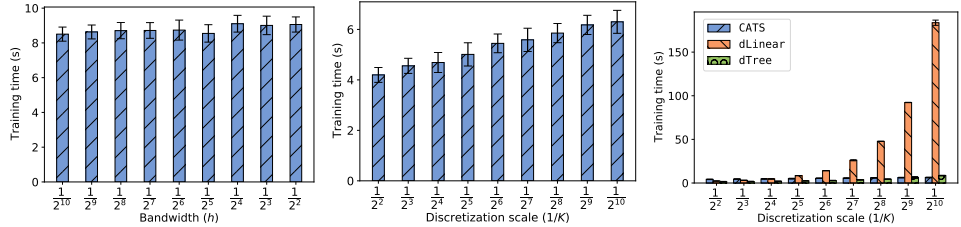

  \begin{center}
\includegraphics[width=0.3\textwidth]{Figures/ds_5_timeh.png}
\includegraphics[width=0.3\textwidth]{Figures/ds_5_timen.png}
\includegraphics[width=0.3\textwidth]{Figures/ds_5_timenn.png}
\end{center}
\caption{Online learning time costs of \cats (blue bar) w.r.t: ({\bf left}) bandwidth ($h$) with a fixed discretization scale $K = 2^{13}$; ({\bf middle}) discretization scale ($1/K$) with a fixed $h = 1/4$; ({\bf right}) discretization scale ($1/K$) with a fixed $h = 1/4$, compared against \dl (orange bar) and \dt (green bar), on the \texttt{ds} dataset.Similar figures for the rest of the datasets can be found in the Appendix~\ref{sec:add-experiments}.}
\label{fig:time}
\end{figure}
\xhdr{Off-policy optimization using \catsop.}
A major advantage of the \cats approach over na\"{i}ve discretization methods is that the interaction log collected by our algorithm with one setting of $(h,K)$ can be used to optimize policies with alternate settings of $(h,K)$. To validate this, we first create an 80-20\% split or training and test sets. With the training set, we first collect interaction log tuples of $(x_t, a_t, P_t(a_t \mid x_t), \ell_t(a_t))$ using \cats with initial discretization and smoothing parameter $(K_{\texttt{init}}, h_{\texttt{init}}) = (4,\frac{1}{4})$, and greedy parameter $\epsilon = 0.05$. We then run \catsop over the logged data using $\Jcal$, defined in~\eqref{eqn:j-expt}, as the set of parameters.
Since standard generalization error bounds are loose in practice, we replaced $64\ln\frac {2T \abr{\Jcal}} \delta$ in the penalty term in line~\ref{line:srm-off} with constant $1$. Note that this constant term as well as the learning rate and the greedy parameter are fixed for all of the datasets in our experiments.

The right panel in Figure~\ref{fig:offline_abs} shows the test losses of the models obtained by \cats after making a pass over the training data, and the test losses of the optimized models obtained through \catsop by optimizing counterfactual estimates offline. It can be seen that offline policy training produces tree policies that have dramatically smaller test losses than the original policies.





\section{Conclusion}
Contextual bandit learning with continuous actions with unknown
structure is quite tractable via the \cats algorithm, as we have shown
theoretically and empirically.  This broadly enables deployment of
contextual bandit approaches across a wide range of new applications.

\section*{Broader Impact}
Our study of
efficient contextual bandits with continuous actions can be applied to
a wide range of applications, such as precision medicine, personalized
recommendations, data center optimization, operating systems,
networking, etc.  Many of these applications have potential for
significant positive impact to society, but these methods can also
cause unintend harms, for example by creating filter bubble effects
when deployed in recommendation engines.  More generally our research
belongs to the general paradigm of interactive machine learning, which
must always be used with care due to the presence of feedback
loops. We are certainly mindful of these issues, and encourage
practitioners to consider these consequences when deploying
interactive learning systems.



\begin{ack}
We thank the anonymous reviewers for their helpful feedback.
Much of this work was done while Maryam Majzoubi and Chicheng Zhang were visiting Microsoft Research NYC.
This work was supported by Microsoft.  
\end{ack}

\bibliographystyle{plain}
\bibliography{refs}

\newpage
\onecolumn

\appendix
\section{Additional Notation}
\label{sec:add-not}
Throughout the appendices, we will use all notation from Section~\ref{sec:prelims}, without further recap, as well as some additional notation presented below. For a policy $\pi$, define $V(\pi) = \lambda_0(\pi) = \EE_{(x, \ell) \sim \Dcal} \EE_{a \sim \pi(\cdot \mid x) } \sbr{\ell(a)}$ to be its expected loss. We will use the notations $V(\pi)$ and $\lambda_0(\pi)$ interchangably throughout the appendix.

For a subset of indices $B \subset [n]$ and a policy $\pi$, denote by $\hat{V}_{B}(\pi_h) = \frac{1}{\abr{B}} \sum_{s \in B} \frac{\pi_{h}(a_s \mid x_s)}{P_s(a_s \mid x_s)} \ell_s(a_s)$.

For a general policy class $\Pi \subset (\Xcal \to \Acal)$, we define the $h$-smoothed regret of an algorithm against $\Pi$ for a time horizon of $T$ as:
\begin{align*}
    \Reg(T,\Pi,h) \defeq \sum_{t=1}^T \EE \sbr{\ell_t(a_t)} - T\inf_{\pi \in \Pi} \lambda_h(\pi) = \sum_{t=1}^T \EE\sbr{\ell_t(a_t)} - T \inf_{\pi \in \Pi} V(\pi_h).
\end{align*}

We will be using the following property of logged data, which has the essential independence structure to guarantee the quality of the model trained with \treetrain on its induced CSMC examples using IPS.

\begin{definition}[Well-formed logged data]
The logged data $\cbr{(x_s, a_s, P_s(a_s \mid x_s), \ell_s(a_s))}_{s=1}^n$ is said to be $p_{\min}$-{\em well-formed}, if it is generated by the following process: $(x_s, \ell_s)_{s=1}^n$ are drawn iid from $\Dcal$, action distribution $P_s(\cdot \mid \cdot)$ depends only on $(x_{s'}, a_{s'}, \ell_{s'})_{s'=1}^{s-1}$, $P_s(a \mid x) \geq p_{\min}$ for all $a \in \Acal$, $x \in \Xcal$, and $s \in [n]$.
\label{def:well-formed}
\end{definition}

A formal description of the execution of tree policies, i.e. $\tree.\getaction(x)$, is given in Algorithm~\ref{alg:filter-tree-pred}.
\begin{algorithm}[t]
\caption{Execution of tree policy $\tree$: $\tree.\getaction$}
\label{alg:filter-tree-pred}
\begin{algorithmic}
\REQUIRE Tree policy $\tree$ using classifiers $\cbr{f^\vt} \subset (\Fcal \to \cbr{\leftt, \rightt})$, context $x$.
\STATE Let $\vt \gets \tree.\roott$. 
\WHILE{$\vt$ is an internal node of $\tree$}
	\STATE $\vt \gets \vt.f^\vt(x)$
\ENDWHILE
\RETURN $a \gets \labelt(\vt)$, the action label of $\vt$.
\end{algorithmic}
\end{algorithm}


\section{Proofs of Theorems~\ref{thm:greedy-tree} and~\ref{thm:off-policy-opt}}
\label{sec:proofs}

In this section, we first prove a key lemma, namely Lemma~\ref{lem:off-policy-fixed}, and use it to show Theorems~\ref{thm:greedy-tree} and~\ref{thm:off-policy-opt} in the main text respectively.

\subsection{Off-policy optimization guarantees on trees with well-formed logged data}

Recall that $\Fcal$ is a class of binary classifiers, and $\Dcal$ is a distribution over (context, loss) pairs.
In words, this lemma states that, under realizability and the well-formedness property of the logged data, training using \treetrain based on its induced IPS CSMC examples yields a tree policy that has a $h$-smoothed loss competitive with any tree policy in tree class $\Fcal_K$.

\begin{lemma}[Off-policy optimization with tree classes under realizability]
Suppose:
\begin{enumerate}
\item $(\Fcal, \Dcal)$ is $(h, K)$-realizable for $h > 0$, $K = 2^D$ for some $D$ in $\NN$.
\item The logged data $\cbr{(x_s, a_s, P_s(a_s \mid x_s), \ell_s(a_s))}_{s=1}^n$  is $p_{\min}$-well-formed.
\end{enumerate}
In addition, Algorithm~\ref{alg:tree-train} is run with dataset $S = \cbr{(x_s, \tilde{c}_s)}_{s=1}^n$ (a set of CSMC examples induced by the logged data using IPS; see Section~\ref{sec:prelims} for the definition of $\tilde{c}_s$), bandwidth $h$, discretization level $K$, base class $\Fcal$.
Then, with probability $1-\delta$, the policy $\tree$ returned is such that:
\begin{eqnarray*}
  V(\tree_h)
  \leq
  \min_{\tree' \in \Fcal_K} V(\tree_h')
    + 20 \sqrt{\frac{K^2 \log K}{n p_{\min} h} \cdot \rbr{\ln\frac{2 n K \abr{\Fcal}}\delta}}
\end{eqnarray*}
\label{lem:off-policy-fixed}
\end{lemma}


\begin{proof}[Proof of Lemma~\ref{lem:off-policy-fixed}]
We will show the following claim: for every node $\vt$ in $\tree$, there exists an event $E_\vt$ that happens with probability at least $1- \nicefrac{\delta \abs{\tree^\vt}}{2K}$, in which
\begin{eqnarray}
\EE \sbr{ \EE[\ell_h (\tree^\vt(x)) \mid x] - \min_{a \in \range(\tree^\vt)} \EE[\ell_h (a) \mid x]}
&=&
\EE \sbr{ \ell_h (\tree^\vt(x)) } - \EE \sbr{\min_{a \in \range(\tree^\vt)} \EE[\ell_h(a) \mid x]} \nonumber \\
&\leq& \abs{\tree^\vt} \rbr{ 8 \sqrt{\frac{\ln\frac{2n'  K |\Fcal|}{\delta}}{n' p_{\min} h}} + 4 \frac{\ln\frac{2n'  K |\Fcal|}{\delta}}{n' p_{\min} h}},
\label{eqn:induct-tree}
\end{eqnarray}
where $\abs{\tree^\vt}$ is the total number of nodes in subtree $\tree^\vt$ (including internal nodes and leaves), and $n' = \frac{n}{\log K}$ is the number of examples for training at each level of $\tree$. As $\tree$ is a complete binary tree with $K - 1$ internal nodes and $K$ leaves,$\abs{\tree^\vt} = 2K - 1$.
To see why it completes the proof, we set $\vt$ to be the root of $\tree$. In this case, we get that with probability $1 - \nicefrac{\delta \abs{\tree}}{2K} \geq 1 - \delta$,
\begin{eqnarray*}
\EE \sbr{ \ell_h(\tree(x)) } - \EE \sbr{\min_{a \in \Acal_K} \EE[\ell_h(a) \mid x]}
\leq (2K-1) \cdot \rbr{ 8 \sqrt{\frac{\ln\frac{2n' K |\Fcal|}{\delta}}{n' p_{\min} h}} + 4 \frac{\ln\frac{2n' K |\Fcal|}{\delta}}{n' p_{\min} h}}.
\end{eqnarray*}
Observing that as $\range(\tree') = \Acal_K$ for all $\tree'$ in $\Fcal_K$, we have that $\EE \sbr{\min_{a \in \Acal_K} \EE[\ell_h(a) \mid x]} \leq  \min_{\tree' \in \Fcal_K} \EE \sbr{ \EE[\ell_h(\tree'(x)) \mid x] } = \min_{\tree' \in \Fcal_K} V(\tree_h)$. In conjunction with the fact that $\frac{n}{2\log K} \leq n' \leq n$, we get that
\begin{eqnarray*}
\EE \sbr{ \ell_h(\tree(x)) } - \min_{\tree' \in \Fcal_K} \EE \sbr{\ell_h(\tree'(x))}
\leq \rbr{ 16 \sqrt{\frac{ K^2 \log K \ln\frac{2n K |\Fcal|}{\delta}}{n p_{\min} h}} + \frac{8 K \log K \ln\frac{2n K |\Fcal|}{\delta}}{n p_{\min} h}}.
\end{eqnarray*}
The lemma follows, because if $\frac{8 K \log K \ln\frac{2n K |\Fcal|}{\delta}}{n p_{\min} h} \geq \frac12$, the lemma statement is trivially true, as the right hand is at least $1$, and the left hand side is at most $1$; otherwise, $\frac{8 K \log K \ln\frac{2n K |\Fcal|}{\delta}}{n p_{\min} h} \leq 4 \sqrt{\frac{K \log K \ln\frac{2n K |\Fcal|}{\delta}}{n p_{\min} h}}$, in which case the right hand side is at most $20 \sqrt{\frac{K^2 \log K \ln\frac{2n K |\Fcal|}{\delta}}{n p_{\min} h}}$.

Next we turn to show the above claim by induction.

\paragraph{Base case.} If $\vt$ is of depth $D-1$, i.e. it is the parent of a pair of leaves $\lt \defeq \vt.\leftt \in \Acal_K$ and $\rt \defeq \vt.\rightt \in \Acal_K$, then
$c^\vt(\leftt) = \tilde{c}(\labelt(\lt))$, $c^\vt(\rightt) = \tilde{c}(\labelt(\rt))$.
In addition, $\range(\tree^\vt) = \cbr{ \labelt(\lt), \labelt(\rt) }$.
 Given a classifier $f: \Xcal \to \cbr{\leftt, \rightt}$ in $\Fcal$, we define its induced tree policy at node $\vt$, $\pi_f: \Xcal \to \Acal_K$, as: $\pi_f(x) = \labelt(\vt.f(x))$.


%
Observe that the CSMC examples $\cbr{(x_s, \tilde{c}_s^h)}_{s \in B_{D-1}}$ (where $B_{D-1} = [n']$) can be viewed as induced by a set of $p_{\min}$-well-formed logged data $\cbr{(x_s, a_s, P_s(a_s \mid x_s), \ell_s(a_s))}_{s \in B_{D-1}}$ using IPS.
From Lemma~\ref{lem:unif-conv-log} in Appendix~\ref{sec:concentration}, we have that there exists an event $E_\vt$ such that $\PP(E_\vt) \geq 1 - \nicefrac{\delta}{K}$, on which for all $f$ in $\Fcal$,
\begin{equation}
 \abs{ \hat{V}_{B_{H-1}}(\pi_{f,h}) - V(\pi_{f,h})}
 \leq
 \rbr{
 	  4 \sqrt{\frac{\ln|\Fcal| + \ln\frac{2n' K}{\delta}}{n' p_{\min} h}} +
      2 \frac{\ln|\Fcal| + \ln\frac{2n' K}{\delta}}{n' p_{\min} h}
     }.
\label{eqn:uc-v-leaf}
\end{equation}
We henceforth condition on $E_\vt$ happening.

Observe that $\hat{V}_{B_{D-1}}(\pi_{f,h}) = \EE_{S^\vt} \sbr{c^\vt(f(x))}$;
As $f^\vt = \argmin_{f \in \Fcal} \EE_{S^\vt} \sbr{c^\vt(f(x))}$, we have that:
$\hat{V}_{B_{D-1}}(\pi_{f^\vt,h}) \leq \hat{V}_{B_{D-1}}(\pi_{f^{\vt,\star},h})$ for $f^{\vt,\star}$
defined in Definition~\ref{def:well-specified-lh}.
This fact, in conjunction with Equation~\eqref{eqn:uc-v-leaf}, gives that
\begin{equation*}
 V(\pi_{f^\vt,h}) - V(\pi_{f^{\vt,\star},h})
 \leq
 \rbr{
 	  8 \sqrt{\frac{\ln|\Fcal| + \ln\frac{2n' K}{\delta}}{n' p_{\min} h}} +
      4 \frac{\ln|\Fcal| + \ln\frac{2n' K}{\delta}}{n' p_{\min} h}
     }.
\end{equation*}


Also, by Definition~\ref{def:well-specified-lh},
$
 V(\pi_{f^{\vt,\star},h})
 = \EE \sbr{ \ell_h (\pi_{f^{\vt,\star}}(x)) }
 = \EE \sbr{ \ell_h (\labelt(\vt.f^{\vt,\star}(x))) }
 = \EE \sbr{
              \min_{ a \in \range(\tree^\vt) } \EE[c(a) \mid x]
           }
$. In addition, by the definition of $\pi_{f}$, $\tree^\vt = \pi_{f^\vt}$.
Therefore,
\[
    \EE [\ell_h (\tree^\vt(x))] - \EE \sbr{\min_{a \in \range(\tree^\vt)} \EE[\ell_h(a) \mid x]}
    \leq
    \rbr{
 	  8 \sqrt{\frac{\ln|\Fcal| + \ln\frac{2n' K}{\delta}}{n' p_{\min} h}} +
      4 \frac{\ln|\Fcal| + \ln\frac{2n' K}{\delta}}{n' p_{\min} h}
     },
\]
proving the base case.

\paragraph{Inductive case.} Suppose that the results holds for all nodes $\vt$ at level $\geq d + 1$. For node $\vt$ at depth $d$, suppose $\lt = \vt.\leftt$ and $\rt = \vt.\rightt$ are its two children at level $d+1$. In this notation,
given an IPS CSMC example $(x, \tilde{c})$ in $B_d$,
$c^\vt(\leftt) = \tilde{c}(\tree^\lt(x))$, $c^\vt(\rightt) = \tilde{c}(\tree^\rt(x))$.
Given a classifier $f: \Xcal \to \cbr{\leftt, \rightt}$ in $\Fcal$, and the subtree policies $\tree^\lt$, $\tree^\rt$,
we define its induced tree policy at $\vt$,  $\pi_f: \Xcal \to \Acal_K$ as: $\pi_f(x) = \tree^{\vt.f(x)}(x)$.


First, consider the training of classifier $f^{\vt}$ at node $\vt$. We note that given logged data with indices
$\cup_{d'=d+1}^{H-1} B_{d'} = [ (H - d - 1) n' ]$ used to learn downstream classifiers in internal nodes of $\tree_\lt$ and $\tree_\rt$, the CSMC examples $\cbr{(x_s, \tilde{c}_s^h)}_{s \in B_l}$ can be viewed as induced by  a set of $p_{\min}$-well-formed logged data $\cbr{(x_s, a_s, P_s(a_s \mid x_s), \ell_s(a_s))}_{s \in B_d}$ using IPS (See Definition~\ref{def:well-formed}). Therefore, applying Lemma~\ref{lem:unif-conv-log}, we get that there exists an event $E_\vt^1$ such that $\PP(E_\vt^1) \geq 1 - \nicefrac{\delta}{K}$, on which for all $f$ in $\Fcal$,
\begin{equation}
 \abs{ \hat{V}_{B_d}(\pi_{f,h}) - V(\pi_{f,h})}
 \leq
 \rbr{
 	  4 \sqrt{\frac{\ln|\Fcal| + \ln\frac{2n' K}{\delta}}{n' p_{\min} h}} +
      2 \frac{\ln|\Fcal| + \ln\frac{2n' K}{\delta}}{n' p_{\min} h}
     }.
\label{eqn:uc-v-internal}
\end{equation}
In addition, by inductive hypothesis, we have that there exists two events $E_\lt$ and $E_{\rt}$, happening with probability $1-\nicefrac{\abr{\tree_\lt} \delta}{2K}$ and $1-\nicefrac{\abr{\tree_\rt} \delta}{2K}$ respectively, in which
\begin{align}
	& \EE \sbr{ \EE[\ell_h(\tree^\lt(x))]|x] - \min_{a \in \range(\tree^\lt)} \EE[\ell_h(a)|x] } \nonumber \\
	&\leq  \abr{\tree_\lt} \rbr{ 8 \sqrt{\frac{\ln|\Fcal| + \ln\frac{2n' K}{\delta}}{n' p_{\min} h}} + 4 \frac{\ln|\Fcal| + \ln\frac{2n' K}{\delta}}{n' p_{\min} h}},
	\label{eqn:bias-l}
\end{align}
and
\begin{align}
	& \EE \sbr{ \EE[\ell_h(\tree^\rt(x))]|x] - \min_{a \in \range(\tree^\rt)} \EE[\ell_h(a)|x] } \nonumber \\
	&\leq  \abr{\tree_\rt} \rbr{ 8 \sqrt{\frac{\ln|\Fcal| + \ln\frac{2n' K}{\delta}}{n' p_{\min} h}} + 4 \frac{\ln|\Fcal| + \ln\frac{2n' K}{\delta}}{n' p_{\min} h}},
	\label{eqn:bias-r}
\end{align}
holds respectively. We define $E_\vt = E_\vt^1 \cap E_\lt \cap E_{\rt}$. By union bound, $\PP(E_\vt) \geq 1 - \nicefrac{\abr{\tree_\vt}}{2K}$. We henceforth condition on $E_\vt$ happening.


First, we note that by Equation~\eqref{eqn:uc-v-internal} and the optimality of $\pi_{f^\vt,h}$, $\hat{V}_{B_d}(\pi_{f^\vt,h}) \leq \hat{V}_{B_d}(\pi_{f^{\vt,\star},h})$ for $f^{\vt,\star}$
defined in Definition~\ref{def:well-specified-lh}.
This fact, in conjunction with Equation~\eqref{eqn:uc-v-internal}, gives that
\begin{align}
 \EE \sbr{ \ell_h(\tree^\vt(x)) } - \EE \sbr{ \ell_h( \tree^{\vt.f^{\vt,\star}(x)}(x) ) }
 &=
 V(\pi_{f^\vt,h}) - V(\pi_{f^{\vt,\star},h}) \nonumber \\
 &\leq
 \rbr{
 	  8 \sqrt{\frac{\ln|\Fcal| + \ln\frac{2n' K}{\delta}}{n' p_{\min} h}} +
      4 \frac{\ln|\Fcal| + \ln\frac{2n' K}{\delta}}{n' p_{\min} h}
     }.
\label{eqn:erm-v-internal}
\end{align}



We have the following inequalities:
\begin{eqnarray}
&& \EE \sbr{\ell_h( \tree^{\vt.f^{\vt,\star}(x)}(x) )} \nonumber \\
&=& \EE \sbr{ \EE[\ell_h (\tree^{\lt}(x))|x] \ind(f^{\vt,\star}(x) = \leftt) + \EE[\ell_h (\tree^{\rt}(x))|x] \ind(f^{\vt,\star}(x) = \rightt) } \nonumber \\
&\leq& \EE \sbr{ \min_{a \in \range(\tree^\lt)} \EE[c(a)|x] \ind(f^{\vt,\star}(x) = \leftt) } + \EE \sbr{ \min_{a \in \range(\tree^\rt)} \EE[c(a)|x] \ind(f^{\vt,\star}(x) = \rightt) } \nonumber \\
&& + (|\tree^\lt| + |\tree^\rt|) \rbr{ 8 \sqrt{\frac{\ln|\Fcal| + \ln\frac{n' K}{\delta}}{n' p_{\min} h}} + 4 \frac{\ln|\Fcal| + \ln\frac{n' K}{\delta}}{n' p_{\min} h}} \nonumber \\
&\leq& \EE \sbr{ \min_{a \in \range(\tree^\vt)} \EE[c(a)|x] } + (|\tree^\lt| + |\tree^\rt|) \rbr{ 8 \sqrt{\frac{\ln|\Fcal| + \ln\frac{2n' K}{\delta}}{n' p_{\min} h}} + 4 \frac{\ln|\Fcal| + \ln\frac{2n' K}{\delta}}{n' p_{\min} h}}.
\label{eqn:bias-all}
\end{eqnarray}
where the first inequality is from Equations~\eqref{eqn:bias-l} and~\eqref{eqn:bias-r}, the second inequality is from the $(h,K)$-realizability assumption.

Therefore, combining Equations~\eqref{eqn:erm-v-internal} and~\eqref{eqn:bias-all}, we get
\begin{align*}
& \EE \sbr{ \ell_h(\tree^\vt(x)) } -  \EE \sbr{ \min_{a \in \range(\tree^v)} \EE[\ell_h(a)|x] } \\
& \leq (1 + |\tree^\lt| + |\tree^\rt|) \rbr{ 8 \sqrt{\frac{\ln|\Fcal| +  \ln\frac{2n' K}{\delta}}{n' p_{\min} h}} + 4 \frac{\ln|\Fcal| + \ln\frac{2n' K}{\delta}}{n' p_{\min} h}}  \\
& = |\tree^\vt| \rbr{ 8 \sqrt{\frac{\ln|\Fcal| + \ln\frac{2n' K}{\delta}}{n' p_{\min} h}} + 4 \frac{\ln|\Fcal| + \ln\frac{2n' K}{\delta}}{n' p_{\min} h}}.
\end{align*}
This completes the induction, and proves the claim.
\end{proof}

\subsection{Proof of Theorem~\ref{thm:greedy-tree}}

We first give a formal statement of Theorem~\ref{thm:greedy-tree} in the theorem below.
\begin{theorem}
Suppose Algorithm~\ref{alg:greedy-tree} is run with greedy parameter $\epsilon$,  smoothing parameter $h$, discretization scale $K$, and base hypothesis class $\Fcal$.
In addition, suppose $(\Fcal, \Dcal)$ is $(h, K)$-realizable.
Then with probability $1-\delta$, it has $h$-smoothed regret against $\Fcal_\infty$ bounded as:
\begin{align*}
&\Reg(T, \Fcal_\infty, h) \leq
\order{\left(\epsilon + \frac{1}{Kh}\right) T + K \sqrt{\frac{T}{\epsilon h} \cdot \rbr{\ln\frac{\abr{\Fcal}}\delta}}}.
\end{align*}
Taking $\epsilon = \rbr{\frac{\ln\frac{\abr{\Fcal}}\delta}{T h^3}}^{\nicefrac{1}{5}}$,
$K = \rbr{\frac{T}{h^2 \ln\frac{\abr{\Fcal}}\delta}}^{\nicefrac{1}{5}}$, we have
$
  \Reg(T, \Fcal_\infty, h)
  \leq
  \order{
  	  \rbr{
	  	  \nicefrac{T^4 \ln\frac{\abr{\Fcal}} \delta}{h^3}
		  }^{\nicefrac{1}{5}}
		}.
$
\label{thm:greedy-tree-formal}
\end{theorem}

\begin{proof}[Proof of Theorem~\ref{thm:greedy-tree-formal}]
We will show that with probability $1-\delta$,
\begin{align*}
&\Reg(T, \Fcal_K, h)
\leq
\order{\epsilon T + K \sqrt{\frac{T}{\epsilon h} \cdot \rbr{\ln\frac{2 T K\abr{\Fcal}}\delta}}},
\end{align*}
to see why this completes the proof, we observe that for any policy $\tree$ in $\Fcal_\infty$, there is a policy $\tree_K$ in $\Fcal_K$, such that $\abs{\tree_K(x) - \tree(x)} \leq \frac{1}{K}$: we can take $\tree_K$ to be a truncation of
$\tree$ that only keeps its top $\log K$ levels. In addition, as $\ell_h$ is $\nicefrac 1 h$-Lipschitz, we have
\[
\EE \sbr{ \ell_h(\tree_K(x)) } \leq \EE \sbr{ \ell_h(\tree(x)) } + \frac{1}{K h}.
\]
This implies that $\min_{\tree \in \Fcal_K} \EE \sbr{ \ell_h(\tree(x)) } \leq \min_{\tree' \in \Fcal_\infty} \EE \sbr{ \ell_h(\tree'(x)) } + \frac{1}{K h}$. As a result,
\[
  \Reg(T, \Fcal_\infty, h)
  \leq \Reg(T, \Fcal_K, h) + \frac{T}{K h}
  = \order{ \rbr{\epsilon + \frac{1}{Kh}} T + K \sqrt{\frac{T}{\epsilon h} \cdot \rbr{\ln\frac{2 T K \abr{\Fcal}}\delta}}}.
\]

We now come back to the proof of the above claim. First observe that the $h$-smoothed regret can be rewritten as:
\begin{align}
\Reg(T, \Fcal_K, h)
&= \sum_{t=1}^T \rbr{\EE \sbr{\ell_t(a_t)} - \min_{\tree' \in \Fcal_K} V(\tree'_h)}.
\end{align}


Let $\pi_{t+1}$ denote the tree $\Tcal$ at the beginning of time step $t+1$, which is learned from CSMC examples $\cbr{(x_s, \tilde{c}_s)}_{s=1}^t$ by \treetrain.
Define event
\[
E = \cbr{ \text{for all time steps $t$ in $[T-1]$, }
  V(\pi_{t+1, h})
  \leq
  \min_{\tree' \in \Fcal_K} V(\tree'_h)
    + 20 \sqrt{\frac{K^2 \log K}{\epsilon h t} \cdot \rbr{\ln\frac{2TK \abr{\Fcal}}\delta}} }.
\]
From Lemma~\ref{lem:off-policy-fixed} with $p_{\min} = \epsilon$, $\delta' = \frac{\delta}{T}$, and a union bound over all $t \in [T]$, we get that $\PP(E) \geq 1-\delta$.

Now, conditioned on event $E$ happening, we conclude the regret bound.
We first have the following upper bound on the algorithm's instantaneous loss at time $t$, namely $\EE \sbr{\ell_t(a_t)}$:
\begin{align}
\EE[\ell_t(a_t)] &= (1-\epsilon) \cdot \EE_{(x_t,\ell_t) \sim D} \EE_{a \sim \pi_{t,h}(\cdot \mid x_t)}[\ell_t(a)] + \epsilon \cdot \EE_{(x_t,\ell_t) \sim D} \EE_{a \sim U(\Acal)}[\ell_t(a)] \nonumber \\
& \leq V(\pi_{t,h}) + \epsilon.
\end{align}

Therefore, for all $t \in \cbr{2,\ldots,T}$, we have
\begin{equation}
\EE[\ell_t(a_t)] \leq \epsilon + \min_{\tree' \in \Fcal_K} V(\tree'_h)
    + 20 \sqrt{\frac{K^2 \log K}{\epsilon h (t-1)} \cdot \rbr{\ln\frac{2TK \abr{\Fcal}}\delta}}
\label{eqn:inst-reg}
\end{equation}

We now conclude the regret bound:
\begin{align*}
\Reg(\Fcal_K, T, h)
& = \sum_{t=1}^T \rbr{ \EE[\ell_t(a_t)] - \min_{\tree' \in \Fcal_K} V(\tree'_h) } \\
&\leq 1 + \epsilon (T-1) + \sum_{t=2}^T 20 \sqrt{\frac{K^2 \log K}{\epsilon h (t-1)} \cdot \rbr{\ln\frac{2T K\abr{\Fcal}}\delta}} \\
&\leq 1 + \epsilon T + 40 \cdot \sqrt{\frac{T K^2 \log K}{\epsilon h} \cdot \rbr{\ln\frac{2T K\abr{\Fcal}}\delta}}.
\end{align*}
where the first inequality uses the fact that $\EE[\ell_t(a_t)] - \min_{\tree' \in \Fcal_K} V(\tree'_h)$ is at most 1 if $t = 1$, and is at most $\epsilon + 20 \sqrt{\frac{K^2 \log K}{\epsilon h (t-1)} \cdot \rbr{\ln\frac{2TK\abr{\Fcal}}\delta}}$ if $t \geq 2$, and the second inequality uses the fact that $\sum_{t=1}^{T-1} \frac{1}{\sqrt{t}} \leq 2\sqrt{T}$. The theorem follows.
\end{proof}

\subsection{Proof of Theorem~\ref{thm:off-policy-opt}}

We first give a formal statement of Theorem~\ref{thm:off-policy-opt} below.

\begin{theorem}\label{thm:off-policy-opt-formal}
Suppose Algorithm~\ref{alg:off-policy-opt} is run with a set of $p_{\min}$-well-formed logged data $\cbr{x_t, a_t, P_t(a_t \mid x_t), \ell_t(a_t)}_{t=1}^T$, set of (bandwidth, discretization) combinations $\Jcal \subset [0,1] \times 2^\NN$, base hypothesis class $\Fcal$.
In addition, suppose $(\Fcal, \Dcal)$ is $(h, K)$-realizable for all $(h,K) \in \Jcal$. Then, with probability $1-\delta$, its returned policy $\hat{\pi}$ ensures:
\begin{align*}
	\lambda_0(\hat{\pi})
	\leq \min_{(h,K) \in \Jcal, \pi \in \Fcal_K} &
	             \rbr{ \lambda_h(\pi) +  \order{K \sqrt{\nicefrac{\ln\frac{\abr{\Fcal} \abr{\Jcal}}\delta}{(p_{\min} h T)}} } }.
\end{align*}
\end{theorem}

\begin{proof}[Proof of Theorem~\ref{thm:off-policy-opt-formal}]
For every $(h,K)$ in $\Jcal$, recall that $\Tcal_t^{h,K}$ denotes the policy trained by \catsop at the beginning of iteration $t$ for that $(h,K)$ combination.

Define events
\begin{align*}
  E_1 =
  \Big\{ \forall (h,K) \in\Jcal, \forall t \in [T-1],
    V(\tree^{h, K}_{t+1, h})
    &  \leq
    \min_{\tree' \in \Fcal_K} V(\tree'_h) \\
    & + 20 \sqrt{\frac{K^2 \log K}{p_{\min} h t} \cdot \rbr{\ln\frac{4TK \abr{\Fcal}\abr{\Jcal}}\delta}}
    \Big\}
\end{align*}
\begin{align*}
  E_2 =
  \Big\{ \forall (h,K) \in\Jcal, \forall t \in [T], &
   \abs{\frac 1 T \sum_{t=1}^T \tilde{c}_t^h(\tree^{h, K}_{t}(x_t)) - \frac 1 T \sum_{t=1}^T V(\tree^{h, K}_{t, h}) } \\
   & \leq
   8 \sqrt{
   \rbr{\frac 1 T \sum_{t=1}^T V(\tree^{h, K}_{t, h})} \cdot
   \frac{\ln\frac{2 T \abr{\Jcal}}\delta}{p_{\min} h T}}
   + 4 \frac{\ln\frac{4 T \abr{\Jcal}}\delta}{p_{\min} h T}
   \Big\}
\end{align*}

From Lemma~\ref{lem:off-policy-fixed} in Appendix~\ref{sec:concentration} and union bound, we know that $\PP(E_1) \geq 1 - \frac{\delta}{2}$; from Lemma~\ref{lem:unif-conv-log}, item~\ref{item:pv} and union bound over all $(h,K) \in \Jcal$, we get that $\PP(E_2) \geq 1 - \frac{\delta}{2}$.
Define event $E \defeq E_1 \cap E_2$. By union bound, $\PP(E) \geq 1 - \delta$. We henceforth condition on event $E$ happening.




We denote $\hat{g}(h, K) \defeq \frac 1 T \sum_{t=1}^T \tilde{c}_t^h(\tree^{h, K}_{t}(x_t))$, $g(h,K) \defeq \frac 1 T \sum_{t=1}^T V(\tree^{h, K}_{t, h})$, $\sigma(h, K) \defeq \frac{64 \ln\frac{4 T \abr{\Jcal}}\delta}{p_{\min} h T}$.
Using this notation, and by the definition of $E_2$, for all $(h,K)$ in $\Jcal$,
\[
\abr{ \hat{g}(h, K) - g(h, K) } \leq \sqrt{ g(h, K) \sigma(h, K)} + \sigma(h, K)
\]
Specifically,
\begin{equation}
g(h, K) \leq \hat{g}(h, K) + \sqrt{\hat{g}(h, K) \sigma(h, K)} + \sigma(h, K),
\label{eqn:hat-g-g}
\end{equation}

In addition, from the elementary fact that $A \leq B + C\sqrt{A} \Rightarrow A \leq B + C^2 + C\sqrt{B}$, we have
\begin{equation}
g(h, K) \leq \hat{g}(h, K) + \sqrt{\hat{g}(h, K) \sigma(h, K)} + 3\sigma(h, K).
\label{eqn:g-hat-g}
\end{equation}

By the optimality of $\hat{h}, \hat{K}$, for all $(h,K)$ in $\Jcal$,
\begin{equation}
  \hat{g}(\hat{h}, \hat{K}) + \sqrt{\hat{g}(\hat{h}, \hat{K}) \sigma(\hat{h}, \hat{K})} + 3 \sigma(\hat{h}, \hat{K})
  \leq
  \hat{g}(h, K) + \sqrt{\hat{g}(h, K) \sigma(h, K)} + 3\sigma(h, K).
  \label{eqn:opt-hat}
\end{equation}

Therefore, we have the following set of inequalities for every $h \in \Hcal$ and $K \in \Kcal$:
\begin{align}
g(\hat{h}, \hat{K}) & \leq \hat{g}(\hat{h}, \hat{K}) + \sqrt{ \hat{g}(\hat{h}, \hat{K}) \sigma(\hat{h}, \hat{K}) } + 3\sigma(\hat{h}, \hat{K}) \nonumber \\
&\leq \hat{g}(h, K) + \sqrt{ \hat{g}(h, K) \sigma(h, K) } + 3\sigma(h, K) \nonumber\\
&\leq g(h, K) + 3 \sqrt{ g(h, K) \sigma(h, K) } + 6 \sigma(h, K)
\label{eqn:g-hat-h-K}
\end{align}
where the first inequality uses Equation~\eqref{eqn:hat-g-g}; the second inequality is from Equation~\eqref{eqn:opt-hat}, the third inequality again uses Equation~\eqref{eqn:hat-g-g} and algebra.

We claim that $g(\hat{h}, \hat{K}) \leq g(h, K) + 9 \sqrt{\sigma(h,K)}$, because If $\sigma(h, K) \geq 1$, the statement is trivially true as 
$g(\hat{h}, \hat{K}) \leq 1$; otherwise, $6 \sigma(h, K) \leq 6 \sqrt{\sigma(h, K)}$, and the RHS of the above inequality is at most $g(h, K) + (3+6) \sqrt{\sigma(h,K)} \leq g(h, K) + 9 \sqrt{\sigma(h,K)}$.

Rephrasing the above inequality using our previous notation, we have:
\begin{equation}
\frac 1 T \sum_{t=1}^T V(\tree^{\hat{h}, \hat{K}}_{t,\hat{h}})
\leq \frac 1 T \sum_{t=1}^T V(\tree^{h, K}_{t, h}) + 72 \sqrt{ \frac{\ln\frac{4T \abr{\Jcal}}\delta}{p_{\min} h T} }.
\label{eqn:srm-result}
\end{equation}

Meanwhile, observe that by the definition of $E_1$, we can bound $\frac 1 T \sum_{t=1}^T V(\tree^{h, K}_{t-1, h})$ as follows:
\begin{align}
\frac 1 T \sum_{t=1}^T V(\tree^{h, K}_{t, h})
&\leq \min_{\tree \in \Fcal_K} V(\tree_h) + \frac 1 T \rbr{ 1 + \sum_{t=1}^{T-1} 44 \sqrt{\frac{K^2 \log K}{p_{\min} h t} \cdot \rbr{\ln\frac{4T K \abr{\Fcal}\abr{\Jcal}}\delta}}} \nonumber \\
&\leq \min_{\tree \in \Fcal_K} V(\tree_h) + \frac 1 T + 88 \sqrt{\frac{K^2 \log K}{p_{\min} h T} \cdot \rbr{\ln\frac{4T K \abr{\Fcal}\abr{\Jcal}}\delta}}
\label{eqn:drop-rel-bound}
\end{align}
where the first inequality uses the simple fact that $V(\tree^{h, K}_{0, h}) \leq 1$; the second inequality uses the algebraic fact that $\sum_{t=1}^{T-1} \frac{1}{\sqrt{t}} \leq 2\sqrt{T}$.



Combining Equations~\eqref{eqn:srm-result} and~\eqref{eqn:drop-rel-bound}, along with some algebra, we get:
\begin{align*}
\frac 1 T \sum_{t=1}^T V(\tree^{\hat{h}, \hat{K}}_{t,\hat{h}}) &\leq \min_{\tree \in \Fcal_K} V(\tree_h) + \frac1T + 88 \sqrt{\frac{K^2 \log K}{p_{\min} h T} \cdot \rbr{\ln\frac{4T K \abr{\Fcal}\abr{\Jcal}}\delta}} + 72 \sqrt{ \frac{\ln\frac{4T \abr{\Jcal}}\delta}{p_{\min} h T}} \\
&\leq \min_{\tree \in \Fcal_K} V(\tree_h) + 160 \sqrt{\frac{K^2 \log K}{p_{\min} h T} \cdot \rbr{\ln\frac{4T K \abr{\Fcal}\abr{\Jcal}}\delta}}.
\end{align*}
The theorem follows by recognizing that the left hand side is $\EE V(\hat{\pi}) = \EE \lambda_0(\hat{\pi})$, where $\hat{\pi}$ is drawn uniformly at random from $\cbr{\tree^{\hat{h}, \hat{K}}_{t,\hat{h}}}_{t=1}^T$.
\end{proof}

\section{\cats with adaptive bandwidth}
\label{sec:adaptive-tree}
As can be seen from Theorem~\ref{thm:greedy-tree}, \cats obtains smoothed regret guarantees with respect to a fixed value of $h$; in practice, as different loss function have different smoothness properties, it would be useful to develop an algorithm that has performance competitive with $\tree_h$ for all $\tree$ in $\Fcal_\infty$ and all $h$ in $(0,1]$ simultaneously. In this section, we develop a variant of \cats, namely Algorithm~\ref{alg:adaptive-tree}, that has such guarantees. Specifically, with appropriate tuning of its greedy parameters, it achieves the following type of high-probability regret guarantee for some function $R$ in terms of bandwidth $h$, number of discretized actions $K$, base class $\Fcal$, time horizon $T$:
\[
\forall h \in [0,1] \centerdot \Reg(T, \Fcal_\infty, h) \leq R(h, K, \abs{\Fcal}, T),
\]
under the realizability assumptions stated in Definition~\ref{def:well-specified-lh}.

At a high level, Algorithm~\ref{alg:adaptive-tree} follows the same outline of Algorithm~\ref{alg:greedy-tree}: it has an $\epsilon$-greedy action selection step (lines~\ref{line:greedy-start-adaptive} to~\ref{line:greedy-end-adaptive}) and has a tree training step (lines~\ref{line:po-tree-start-adaptive} to~\ref{line:po-tree-end-adaptive}). A crucial difference between Algorithm~\ref{alg:adaptive-tree} and Algorithm~\ref{alg:greedy-tree} is that, it now maintains $\abr{\Hcal}$ policies $\cbr{\tree_{t}^h}_{h \in \Hcal}$ over time as opposed to only one; to this end, it accumulates $\abr{\Hcal}$ CSMC datasets $\cbr{ \cbr{(x_s, \tilde{c}_s^h)}_{s=1}^t}_{h \in \Hcal}$. After generating policies $\cbr{\tree_{t}^h}_{h \in \Hcal}$, it selects $\tree_{t}^{h_t}$ using structural risk minimization~\cite{vapnik1995nature} (line~\ref{line:srm}). This choice of $h_t$ ensures that the expected loss of $\tree_{t,h_t}^{h_t}$ is competitive with all $\tree_{t,h}^h$'s. Finally, we remark that the set of bandwidth $\Hcal$ acts as a covering of the $[0,1]$ interval; as we will see, setting $\Hcal$ to be a fine grid as in Algorithm~\ref{alg:adaptive-tree} ensures that for any $\tree$ in $\Fcal_K$, and every $h$ in $[0,1]$, there exists a $h'$ in $\Hcal$ such that the optimal $\tree_{h'}$ has expected loss close to that of $\tree_h$.


\begin{algorithm}[h]
\caption{\cats with adaptive bandwidth}
\label{alg:adaptive-tree}
	\begin{algorithmic}[1]
    \REQUIRE{Greedy parameter $\epsilon$, number of discretized actions $K = 2^D$, base class $\Fcal$.}
    \STATE Let $\Hcal = \cbr{ h \in \cbr{\frac1{4T^2}, \frac{2}{4T^2}, \ldots, 1}: h \geq \frac1{2T} }$ be the set of bandwidths in consideration.
    \STATE Let $\pi_t$ be an arbitrary policy in $\Fcal_K$.
	\FOR{$t = 1,2,\ldots, T$}	
    \STATE Define policy $P_t(a \mid x) := (1-\epsilon) \pi_t(a|x) + \epsilon$.
    \label{line:greedy-start-adaptive}
    \STATE Observe context $x_t$, select action $a_t \sim P_t(\cdot \mid x_t)$, observe cost $\ell_t(a_t)$.
    \label{line:greedy-end-adaptive}

	\FOR{all $h$ in $\Hcal$}
	\label{line:po-tree-start-adaptive}
	 \STATE $\tilde{c}_t^h(\nicefrac{i}{K}) \gets \frac{\smooth_h(a_t \mid \nicefrac{i}{K})}{P_t(a_t \mid x_t)} \ell_t(a_t)$ for all $i$.
    \STATE Let $\tree^h \gets \treetrain(\cbr{(x_s, \tilde{c}_s^h)}_{s=1}^t)$.
	\ENDFOR
	\label{line:po-tree-end-adaptive}
	
	\STATE Let $h_t \in \argmin_{h \in \Hcal} \rbr{ \hat{V}_t(\tree_{h}^h) + 4 \sqrt{\frac{K \ln|\Fcal| + \ln\frac {8T^4} \delta}{t \epsilon h}} + 2 \frac{K \ln|\Fcal| + \ln\frac {8T^4} \delta}{t \epsilon h}}$, and let $\pi_{t+1} = \tree_{{h_t}}^{h_t}$.
	\label{line:srm}
	
    \ENDFOR
  \end{algorithmic}
\end{algorithm}

%

We next present a theorem on the regret guarantee of Algorithm~\ref{alg:adaptive-tree}.

\begin{theorem}
Suppose Algorithm~\ref{alg:adaptive-tree} is run with greedy parameter $\epsilon$, number of discretized actions $K$, and base class $\Fcal$.
In addition, suppose $(\Fcal, \Dcal)$ satisfies the $(h,K)$-realizability assumption for all $h \in (0,1)$.
Then with probability $1-\delta$, it has uniform $h$-smoothed regret bounded as:
\begin{align*}
\forall h \in [0,1] \centerdot
&\Reg(T, \Fcal_\infty, h) \leq
\otil{(\epsilon + \frac{1}{Kh}) T + \sqrt{\frac{K^2 \log K \cdot T \cdot \rbr{\ln\frac{\abr{\Fcal}}\delta}}{\epsilon h}} }.
\end{align*}
Specifically, by taking $\epsilon = \rbr{\frac{\ln\frac{\abr{\Fcal}}\delta}{T}}^{\nicefrac{1}{5}}$,
$K = \rbr{\frac{T}{\ln\frac{\abr{\Fcal}}\delta}}^{\nicefrac{1}{5}}$, we have
\[
  \forall h \in [0,1] \centerdot
  \Reg(T, \Fcal_\infty, h)
  \leq
  \otil{
  	  \frac 1 h \cdot \rbr{
	  	  T^4 \ln\frac{\abr{\Fcal}} \delta
		  }^{\nicefrac{1}{5}}
		}.
\]
\label{thm:adaptive-tree}
\end{theorem}

Before going into the proof of the theorem, we remark that the only difference between the above regret guarantee of Algorithm~\ref{alg:adaptive-tree} and that of \cats (Theorem~\ref{thm:greedy-tree}) is that, the order of $h$ is different ($\frac1h$ versus $\frac{1}{h^{3/5}}$). This can be seen as a price we pay for adaptivity: Algorithm~\ref{alg:adaptive-tree} sets $K$ independent of $h$, whereas \cats can set $K$ that depends on $h$.

\begin{proof}[Proof sketch]
By standard analysis on structural risk minimization~\citep[see e.g.][]{vapnik1995nature}, and union bound, it can be shown that with probability $1-\delta/2$, for all time steps $t$ in $[T]$ and all $h \in \Hcal$,
\[ V(\tree^{h_t}_{t, h_t}) \leq V(\tree^{h}_{t, h}) + \order{\sqrt{\frac{K \ln\frac{T \abr{\Fcal}}{\delta}}{\epsilon h t }}}. \]

On the other hand, from Lemma~\ref{lem:off-policy-fixed} and union bound over all time steps $t$ in $[T]$, we have that with probability $1-\delta/2$,
\[ V(\tree^{h}_{t, h}) \leq \min_{\tree \in \Fcal_K} V(\tree_K) + \order{\sqrt{\frac{K^2 \log K \cdot \rbr{\ln\frac{T \abr{\Fcal}}\delta}}{\epsilon h t}}}.
\]
Combining the above two inequalities, we have that with probability $1-\delta$, for all $h$ in $\Hcal$,
\[
V(\tree^{h_t}_{t, h_t}) \leq \min_{\tree \in \Fcal_K} V(\tree_K) + \order{\sqrt{\frac{K^2 \log K \cdot \rbr{\ln\frac{T \abr{\Fcal}}\delta}}{\epsilon h t}}}.
\]

By the setting of $\Hcal$, we can guarantee that the above also implies that the equation above holds for all $h \in (0,1]$ (see~\citet[][Lemma 20]{Krish2019colt} for a detailed argument).
By standard regret analysis of $\epsilon$-greedy exploration, this implies that for all $h \in (0, 1]$,
\[
\Reg(T, \Fcal_K, h)
\leq
\epsilon T + \order{\sum_{t=1}^T  \sqrt{\frac{K^2 \log K \cdot \rbr{\ln\frac{T\abr{\Fcal}}\delta}}{\epsilon h t}} }
=
\order{\epsilon T +  \sqrt{\frac{K^2 \log K \cdot T \cdot \rbr{\ln\frac{T\abr{\Fcal}}\delta}}{\epsilon h}} }.
\]
We conclude the first item, by the above inequality, and observing that for any tree policy in $\Fcal_\infty$, there exists a tree policy in $\Fcal_K$ that has extra $h$-smoothed expected loss at most $\frac{1}{hK}$.

The second item follows directly by the settings of $\epsilon$, $K$ and algebra.
\end{proof}


\section{Algorithms for general policy classes}
\label{sec:general}
In this section, we generalize \cats and propose two algorithms, namely Algorithms~\ref{alg:greedy} and~\ref{alg:adaptive}, that works with general policy classes $\Pi$. On one hand, the two algorithms presented in this section may not be computationally efficient in general, because off-policy optimization w.r.t $\Pi$ can be computationally intractable; on the other hand, they have similar regret guarantees as \cats and Algorithm~\ref{alg:adaptive-tree} while being able to handle policy classes beyond trees.

We first present Algorithm~\ref{alg:greedy}, an algorithm that naturally generalizes the $\epsilon$-greedy algorithm~\citep[e.g.][]{Langford-nips07} in the discrete action space setting to the continuous action space setting. It has two input parameters: a bandwidth parameter $h$, and a parameter $\eps \in [0,1]$ that controls the exploration-exploitation tradeoff.


\begin{algorithm}
\caption{Smoothed $\eps$-greedy algorithm with general policy classes}
\label{alg:greedy}
	\begin{algorithmic}[1]
    \STATE Input: Greedy parameter $\epsilon$, smoothing parameter $h$, policy class $\Pi$.
    \STATE Let $\pi_1$ be an arbitrary policy in $\Pi$.
    \FOR{$t=1,2,\ldots$}

    \STATE Define policy $P_t(a|x) := (1-\epsilon) \pi_{t,h}(a|x) + \epsilon$.

    \STATE Observe context $x_t$, select action $a_t \sim P_t(\cdot|x_t)$, observe loss $\ell_t(a_t)$.


    \STATE Find $\pi_{t+1} \gets \argmin_{\pi \in \Pi} \hat{V}_t(\pi_h)$, 
    where
    \[ \hat{V}_t(\pi_h) := \frac{1}{t} \sum_{s=1}^{t} \frac{\pi_h(a_s|x_s)}{P_s(a_s|x_s)} \ell_s(a_s). \]

    \label{line:po-greedy}
    \ENDFOR
  \end{algorithmic}
\end{algorithm}

As we will see, given bandwidth parameter $h$, the algorithm provides a $h$-smoothed regret guarantee. Furthermore, if $\epsilon$ is large, the algorithm explores more, and learns more on the loss function at each round; in contrast, a choice of small $\epsilon$ lets the algorithm focuses more on exploitation, i.e. utilizing the learned policy more extensively.

The algorithm proceeds in rounds.
At round $t$, it generates a stochastic policy $P_t$ that is a mixture of $\pi_{t,h}$ and the uniform distribution, where the mixture weights are $(1-\epsilon)$ and $\epsilon$ respectively. Based on this policy, the algorithm selects an action $a_t \sim P_t(\cdot|x_t)$.
After action $a_t$ is taken, the algorithm observes its loss incurred $\ell_t(a_t)$ and add the tuple $(x_t, a_t, P_t(a_t \mid x_t), \ell_t(a_t))$ into the interaction log.
Then, it uses the interaction log collected up to round $t$ to build policy loss estimators $\hat{V}_{t}(\pi_h)$ for every policy $\pi$ in $\Pi$, which serves a proxy of $\pi_h$'s expected loss $\lambda_h (\pi)$. Then, it finds policy $\pi_{t+1}$ that minimizes $\hat{V}_t(\pi_h)$.
The rationale is that, as $\hat{V}_t(\pi_h)$ concentrates around $\lambda_h (\pi)$ for all $\pi$,
$\pi_{t+1}$ will also approximately minimize $\lambda_h(\cdot)$ among all policies in $\Pi$.

We have the following theorem that characterizes the $h$-smoothed regret of Algorithm~\ref{alg:greedy}.

\begin{theorem}
Suppose Algorithm~\ref{alg:greedy} is run with greedy parameter $\epsilon$, smoothing parameter $h$ and policy class $\Pi$. Then with probability $1-\delta$, it has $h$-smoothed regret bounded as:
\begin{equation*}
\Reg(T, \Pi, h) \leq \otil{ \epsilon T + \sqrt{\frac{T}{\epsilon h} \cdot \rbr{\ln\abr{\Pi} + \ln\frac1\delta}}}.
\end{equation*}
Furthermore, setting $\epsilon= \min\rbr{1, \rbr{\frac{\ln\abr{\Pi} + \ln\frac1\delta}{h T}}^{\frac13}}$, we have that
\begin{eqnarray*}
\Reg(T, \Pi, h)
&\leq& \otil{\rbr{\frac{T^2}{h} \rbr{\ln\abr{\Pi} + \ln\frac1\delta}}^{\frac13}  + \sqrt{\frac{T}{h} \cdot \rbr{\ln\abr{\Pi} + \ln\frac1\delta}}}.
\end{eqnarray*}
\label{thm:greedy}
\end{theorem}

The above theorem gives a regret bound or order $\rbr{\frac{T^2}{h}\ln\abr{\Pi}}^{\frac13}$, which is similar to the $\rbr{T^2 K \ln\abr{\Pi}}^{\frac13}$ regret bound by $\epsilon$-greedy algorithms obtained in the discrete $K$-action setting~\citep[See e.g.][]{Langford-nips07}. Intuitively, $\frac1h$ characterizes the difficulty of obtaining a $h$-smoothed regret guarantee, which serves as the counterpart of the action set size in the discrete action setting.

The most computationally expensive step of Algorithm~\ref{alg:greedy} is line~\ref{line:po-greedy}, where we find the policy $\pi$ in $\Pi$ that has the smallest IPS loss $V_t(\pi)$. As discussed in Section~\ref{sec:prelims}, if $\Pi$ consists of policies that takes actions in the discrete set
$\cbr{\nicefrac{i}{K}}_{i=0}^{K-1}$, the policy optimization problem can be cast as a CSMC problem, where heuristic algorithms that perform approximate ERM abound; indeed, the \treetrain procedure in \cats can be viewed as one such algorithm.



\begin{proof}[Proof of Theorem~\ref{thm:greedy}]
We let $\pi_\star = \argmin_{\pi \in \Pi} V(\pi_h)$ denote the optimal policy in $\Pi$ after $h$-smoothing.
In this notation, recall that the $h$-smoothed regret can be written as:
\begin{align}
\Reg(\Pi, T, h)
&= \sum_{t=1}^T \rbr{\EE \sbr{\ell_t(a_t)} - V(\pi_{\star, h})}.
\label{eqn:sreg-alt-def-pi}
\end{align}



Define event
\[
E = \cbr{ \text{for all $t$ in $[T]$ and all $\pi$ in $\Pi$}, \abr{ \hat{V}_t (\pi_h) - V (\pi_h)} \leq 8 \sqrt{\frac{\ln\frac{2T \abr{\Pi}}{\delta}}{t \epsilon h}} + 4 \frac{\ln\frac{2T \abr{\Pi}}{\delta}}{t \epsilon h}. }
\]
Using Lemma~\ref{lem:unif-conv-log} with $\delta' = \frac{\delta}{T}$ for every $t = 1,2,\ldots, T$, $p_{\min} = \epsilon$, along with union bound over all $t$'s in $[T]$, we get that $\PP(E) \geq 1-\delta$.
We condition on event $E$ happening in the sequel.
We first provide an excess loss bound for policy $\pi_{t,h}$. At time step $t+1$, $\pi_{t+1,h}$ is an empirical risk minimizer, therefore:
\begin{equation}
\hat{V}_t(\pi_{t+1,h}) \leq \hat{V}_t(\pi_{\star,h}).
\label{eqn:approx-erm-star}
\end{equation}
Hence,
\begin{align}
V(\pi_{t+1,h}) &\leq \hat{V}_t(\pi_{t+1,h}) + 8 \sqrt{\frac{\ln\frac{2T \abr{\Pi}}{\delta}}{t \epsilon h}} + 4 \frac{\ln\frac{2T \abr{\Pi}}{\delta}}{t \epsilon h} \nonumber \\
& \leq \hat{V}_t(\pi_{\star,h})+ 8 \sqrt{\frac{\ln\frac{2T \abr{\Pi}}{\delta}}{t \epsilon h}} + 4 \frac{\ln\frac{2T \abr{\Pi}}{\delta}}{t \epsilon h} \nonumber \\
& \leq V(\pi_{\star,h}) + 16 \sqrt{\frac{\ln\frac{2T \abr{\Pi}}{\delta}}{t \epsilon h}} + 8 \frac{\ln\frac{2T \abr{\Pi}}{\delta}}{t \epsilon h} \nonumber,
\end{align}
where the first inequality is from the definition of $E$, and $\pi_t \in \Pi$; the second inequality is from Equation~\eqref{eqn:approx-erm-star}; the third inequality is from the definition of $E$, and $\pi_\star \in \Pi$; 

We now claim that 
\begin{equation} 
V(\pi_{t+1,h}) \leq V(\pi_{\star,h}) + 24 \sqrt{\frac{\ln\frac{2T \abr{\Pi}}{\delta}}{t \epsilon h}}.
\label{eqn:pi_t}
\end{equation}
This is from a standard case analysis, and the simple fact that $V(\pi_{t+1,h}) \leq 1$: if $ \frac{\ln\frac{2T \abr{\Pi}}{\delta}}{t \epsilon h} \geq 1$ the inequality is trivial; otherwise, $16 \sqrt{\frac{\ln\frac{2T \abr{\Pi}}{\delta}}{t \epsilon h}} + 8 \frac{\ln\frac{2T \abr{\Pi}}{\delta}}{t \epsilon h} \leq (16 + 8) \sqrt{\frac{\ln\frac{2T \abr{\Pi}}{\delta}}{t \epsilon h}} = 24 \sqrt{\frac{\ln\frac{2T \abr{\Pi}}{\delta}}{t \epsilon h}}$. 

We now conclude the regret bound. We first have the following upper bound on the algorithm's instantaneous loss $\EE \ell_t(a_t)$:
\begin{align}
\EE[\ell_t(a_t)] &= (1-\epsilon) \EE_{(x_t,\ell_t) \sim D} \EE_{a \sim \pi_{t,h}(\cdot \mid x_t)}[\ell_t(a)] + \epsilon \EE_{(x_t,\ell_t) \sim D} \EE_{a \sim U(\Acal)}[\ell_t(a)] \nonumber \\
& \leq V(\pi_{t,h}) + \epsilon.
\label{eqn:ell-pi_t}
\end{align}

Combining Equations~\eqref{eqn:sreg-alt-def-pi},~\eqref{eqn:pi_t},~\eqref{eqn:ell-pi_t}, along with algebra, we have:
\begin{align*}
\Reg(\Pi, T, h) &\leq \sum_{t=1}^T \rbr{ \epsilon + V(\pi_{t,h}) - V(\pi_{\star,h}) } \\
&\leq \epsilon T + 1 + \sum_{t=2}^T \rbr{ 24 \sqrt{\frac{\ln\frac{2T \abr{\Pi}}{\delta}}{(t-1) \epsilon h}}} \\
&\leq \epsilon T + 1 + 48 \sqrt{ \frac{T \ln\frac{2T\abr{\Pi}}{\delta}}{\epsilon h }}.
\end{align*}
The theorem follows.
\end{proof}

We next present Algorithm~\ref{alg:adaptive}, which achieves $h$-smoothed regret guarantees against $\Pi$ for all $h$ in $(0,1]$ {\em simultaneously}.
It has the following key differences from Algorithm~\ref{alg:greedy}:
\begin{enumerate}
\item Instead of working with a fixed bandwidth $h$, it works with a set of bandwidths $\Hcal$ that provides a covering of the set of bandwidths $(0,1]$ we compete with.

\item Instead of finding a policy $\pi$ that minimizes $\hat{V}_t(\pi_h)$ for a fixed $h$, the algorithm first finds a minimizer of $\hat{V}_t(\pi_h)$ for every $h \in \Hcal$ (namely $\pi_{t+1,h}$), and selects $\pi_{t+1}$ among the set $\cbr{\pi_{t+1,h}}_{h \in \Hcal}$, using a structural risk minimization~\cite{vapnik1995nature} procedure (line~\ref{line:srm}).
Specifically, the choice of $h_{t+1}$ ensures that the expected loss of $\pi_{t+1,h_{t+1}}$ has competitive performance compared with those of the $\pi_{h}$'s, for all $\pi$ in $\Pi$ and all $h$ in $\Hcal$. Here, the bandwidth-dependent penalty term $P(t,h) \defeq 2\sqrt{\frac{\ln|\Pi| + \ln\frac {8T^4} \delta}{t \epsilon h}} + 3\frac{\ln|\Pi| + \ln\frac {8T^4} \delta}{t \epsilon h}$ is crucial, as it accounts for the different concentration rates from $\hat{V}_t(\pi_{t+1,h})$ to $V(\pi_{t+1,h})$ form different values of $h$.
\end{enumerate}

\begin{theorem}
Suppose Algorithm~\ref{alg:adaptive} is run with greedy parameter $\epsilon$ and policy class $\Pi$. Then with probability $1-\delta$, the algorithm has smoothed regret guarantee simultaneously for all $h \in (0,1]$:
\begin{equation*}
\Reg(T, \Pi, h) \leq \otil{ \epsilon T + \sqrt{\frac{T}{\epsilon h} \cdot \rbr{\ln\abr{\Pi} + \ln\frac1\delta}}}.
\end{equation*}
Furthermore, setting $\epsilon = \min\rbr{1, \rbr{\frac{\ln\abr{\Pi} + \ln\frac1\delta}{T}}^{\frac13}}$, we have that for all $h \in (0,1]$:
\begin{eqnarray*}
\Reg(T, \Pi, h)
&\leq& \otil{\frac{\rbr{T^2 \rbr{\ln\abr{\Pi} + \ln\frac1\delta}}^{\frac13}}{\sqrt{h}}}.
\end{eqnarray*}
\label{thm:adaptive}
\end{theorem}

Before proving the theorem, we make two important remarks:
\begin{enumerate}
\item Theorem 12 of~\cite{Krish2019colt} shows that a combination of  \corral\cite{agarwal2016corralling} with \expf~\cite{auer2002nonstochastic}, using an appropriate tuning of learning rate, can obtain a
 uniform-$h$-smoothed regret of the same order, i.e. $\order{\frac{T^{\nicefrac 2 3} \ln\abr{\Pi}^{\frac{1}{3}}}{\sqrt{h}}}$. However, their algorithm requires explicit enumeration of policies from policy class $\Pi$; in contrast, our algorithm can be reduced to a sequence of policy optimization problems, which can admit much more efficient implementations.

\item The above uniform-$h$-smoothed regret rate in terms of $h$ and $T$, i.e. $\order{\frac{T^{\nicefrac 2 3}}{\sqrt{h}}}$, is unimprovable in general, and is therefore {\em Pareto optimal}. This can be seen from the following result from~\citep[Theorem 11]{Krish2019colt}: there exists a continuous-action CB problem with action space $[0,1]$, constants $c, T_0 > 0$, such that for any algorithm and any $T \geq T_0$, there exist two bandwidths $h_1 = \Theta(1)$ and $h_2 = o(1)$\footnote{subject to $T \to \infty$.} such that
$\Reg(T, \Pi, h_1) > \frac{c T^{\nicefrac 2 3}}{\sqrt{h_1}}$ or
$\Reg(T, \Pi, h_2) > \frac{c T^{\nicefrac 2 3}}{\sqrt{h_2}}$.
As a result, for any $\alpha > 0$, designing an algorithm that obtains a uniform-$h$-smoothed-regret guarantee of order $\order{\frac{T^{\frac23-\alpha}}{h^{\frac12}}}$ or order $\order{\frac{T^{\frac23}}{h^{\frac12-\alpha}}}$ is impossible. This result is perhaps surprising, as it shows that an $\epsilon$-greedy algorithm, well known to have suboptimal regret guarantees in the discrete action CB setting, possesses certain optimality properties in the continuous action CB setting, with appropriate modifications.
\end{enumerate}

\begin{proof}[Proof sketch]
By standard analysis on structural risk minimization~\citep[see e.g.][]{vapnik1995nature}, it can be shown that with high probability, for all $h \in \Hcal$:
\[ V(\pi_{t+1,h_{t+1}}) \leq  \min_{\pi \in \Pi} V(\pi_h) + \order{\sqrt{\frac{\ln\frac{T \abr{\Pi}}{\delta}}{t \epsilon h}}}. \]

By the setting of $\Hcal = \cbr{ h \in \cbr{\frac1{4T^2}, \frac{2}{4T^2}, \ldots, 1}: h \geq \frac1{2T} }$, we can show that that the above guarantee implies that the equation above holds for all $h \in (0,1]$; see~\citet[][Lemma 20]{Krish2019colt} for a detailed proof.

By standard regret analysis of $\epsilon$-greedy algorithms and the above upper bound on the instantenous loss of $\pi_{t+1,h_{t+1}}$, we get that
\[
\Reg(T, \Pi, h)
\leq
\epsilon T + \order{\sum_{t=1}^T \sqrt{\frac{\ln\frac{T \abr{\Pi}}{\delta}}{t  \epsilon h}}}
=
\order{\epsilon T + \sqrt{T \frac{\ln\frac{T\abr{\Pi}}{\delta}}{\epsilon  h}}}.
\]
The second item follows directly by the setting of $\epsilon$ and algebra.
\end{proof}


\begin{algorithm}
	\caption{A Pareto-optimal adaptive-$h$ algorithm}
  \label{alg:adaptive}
	\begin{algorithmic}[1]
    \STATE Input: Greedy parameter $\epsilon$, policy class $\Pi$.
    \STATE Let $\Hcal = \cbr{ h \in \cbr{\frac1{4T^2}, \frac{2}{4T^2}, \ldots, 1}: h \geq \frac1{2T} }$ be the set of bandwidths in consideration.

    \STATE Let $\pi_1$ be an arbitrary policy in $\Pi$, and $h_1$ be an arbitrary number in $\Hcal$.

    \FOR{$t=1,2,\ldots,T$}

    \STATE Define policy $P_t(a|x) := (1-\epsilon) \pi_{t,h_t}(a|x) + \epsilon$.

    \STATE Observe context $x_t$, select action $a_t \sim P_t(\cdot|x_t)$, observe loss $\ell_t(a_t)$.


	\STATE For every $h$ in $\Hcal$, compute $\pi_{t+1}^h \in \Pi$ such that
	\begin{equation}
	\hat{V}_t(\pi^h_{t+1,h}) \leq \min_{\pi \in \Pi} \hat{V}_t(\pi^h_{t,h}),
	\end{equation}
	where
	\[ \hat{V}_t(\pi_h) \defeq \frac{1}t \sum_{s=1}^t \frac{\pi_h(a_s|x_s)}{P_s(a_s|x_s)} \ell_s(a_s). \]
	\label{line:po-adaptive}

    \STATE Select $\pi_{t+1} = \pi_{t+1}^{h_{t+1}}$, where
    \begin{equation*}
    h_t \in \argmin_{h \in \Hcal} \rbr{ \hat{V}_t(\pi^h_{t+1, h}) + P(t,h)}.
    \end{equation*}
    where
    \[
    P(t,h) \defeq 2\sqrt{\frac{\ln|\Pi| + \ln\frac {8T^4} \delta}{t \epsilon h}} + 3\frac{\ln|\Pi| + \ln\frac {8T^4} \delta}{t \epsilon h}.
    \]
     \label{line:srm-adaptive}

    \ENDFOR
  \end{algorithmic}
\end{algorithm}




\section{Concentration inequalities}
\label{sec:concentration}

We first recall a well-known variant of Freedman's inequality~\cite{freedman1975tail,bartlett2008high} that is useful to establish our policy evaluation concentration bounds.
\begin{lemma}[See~\cite{bartlett2008high}, Lemma 2]
Suppose $X_1, \ldots, X_n$ is a martingale difference sequence adapted to filtration $\cbr{\Bcal_i}_{i=0}^n$, where $\abs{X_i} \leq M$ almost surely. Denote by $V_n = \sum_{j=1}^n \EE\sbr{X_j^2 \mid \Bcal_{j-1}}$. Then for any constant $\delta \in (0,\frac1e)$,
with probability $1-\delta$,
\begin{equation}
\abs{\sum_{i=1}^n X_i} \leq 4 \sqrt{V_n \ln\frac{2n}{\delta}} + 2 M \ln\frac{2n}\delta.
\label{eqn:freedman}
\end{equation}
\label{lem:freedman}
\end{lemma}
\begin{proof}
Lemma 2 of ~\cite{bartlett2008high} states that for any $\delta' \in (0,\frac 1 e)$, with probability $1-\delta' \cdot \log n$,
\[ \sum_{i=1}^n X_i \leq \max\rbr{ 4\sqrt{ V_n \ln \frac 1 {\delta'}   }, 2 \ln \frac 1 {\delta'} } \]

Letting $\delta' = \frac{\delta}{2\log n}$, we have that with probability $1-\delta/2$,
\[ \sum_{i=1}^n X_i \leq \max\rbr{ 4\sqrt{ V_n \ln \frac 1 {\delta'}   }, 2 \ln \frac 1 {\delta'} } \leq 4 \sqrt{V_n \ln\frac{2n}{\delta}} + 2 M \ln\frac{2n}\delta,   \]
where the second inequality is by algebra and the fact that $\log n \leq n$.
Similarly, by considering random variable $\cbr{-X_i}_{i=1}^n$, we have that with probability $1-\delta/2$,
\[ \sum_{i=1}^n X_i \geq - \rbr{ 4 \sqrt{V_n \ln\frac{2n}{\delta}} + 2 M \ln\frac{2n}\delta} ,   \]
The lemma is concluded by union bound.
\end{proof}

The above lemma implies the following important concentration result on off-policy evaluation and optimization. First we set up some notations.


Suppose logged data $\cbr{ (x_s, a_s, P_s(a_s \mid x_s), \ell_s(a_s)) : s \in [t]}$ is $p_{\min}$-well-formed (recall Definition~\ref{def:well-formed}).
Define a filtration $\cbr{\Bcal_s}_{s=0}^t$ as follows: for all $s \in \cbr{0, 1,\ldots,t}$, $\Bcal_s \defeq \sigma(x_1, a_1, \ell_1, \ldots, x_{s}, a_{s}, \ell_{s})$.
A sequence of random variables $\cbr{Z_s}_{s=1}^t$ is said to be {\em predictable} w.r.t. filtration $\cbr{\Bcal_{s}}_{s=1}^t$ if $Z_s$ is $\Bcal_{s-1}$-measurable. Using the above notation, we see that the sequence of logging policies $\cbr{P_s}_{s=1}^t$ is predictable wrt $\cbr{\Bcal_{s}}_{s=0}^t$.
Lastly, recall from Section~\ref{sec:prelims} that $\tilde{c}_s^h(\nicefrac i K) = \frac{\smooth_{h}(a_s \mid \nicefrac i K)}{P_s(a_s \mid x_s)} \ell_s(a_s)$ for $i \in \cbr{0,1,\ldots, K-1}$, and therefore, $\tilde{c}_s^h(\pi(x_s)) = \frac{\pi_h(a_s \mid x_s)}{P_s(a_s \mid x_s)} \ell_s(a_s)$.

\begin{lemma}
Suppose the setting is described as above. Then,
\begin{enumerate}
\item With probability $1-\delta'$, we have that for any sequence of policies $\cbr{\pi_s}_{s=1}^t$ predictable w.r.t. $\cbr{\Bcal_{s}}_{s=0}^t$,
\begin{equation}
\abr{ \frac 1 t \sum_{s=1}^t \tilde{c}_s^h(\pi_s(x_s)) - \frac 1 t \sum_{s=1}^t V(\pi_{s,h}) }
\leq
\sqrt{ \rbr{\frac 1 t \sum_{s=1}^t V(\pi_{s,h})} \frac{16\ln\frac{2t}{\delta'}}{t \; p_{\min} \; h}} + \frac{2\ln\frac{2t}{\delta'}}{t \; p_{\min} \; h}.
\label{eqn:pv-conv}
\end{equation}
\label{item:pv}
\item Given a finite set of policies $\Pi$, with probability $1-\delta'$, for all $\pi$ in $\Pi$,
\begin{equation}
\abr{ \hat{V}_t (\pi_{h}) - V (\pi_{h}) }
\leq
4 \sqrt{\frac{\ln \abr{\Pi} + \ln\frac{2 t}{\delta'}}{t \; p_{\min} \; h}} + 2 \frac{ \ln \abr{\Pi} + \ln\frac{2 t}{\delta'}}{t \; p_{\min} \; h}.
\label{eqn:unif-conv-log}
\end{equation}
\label{item:uc}
\end{enumerate}
\label{lem:unif-conv-log}
\end{lemma}
\begin{proof}
For the first item, we define
$X_s \defeq \frac{\pi_{s,h}(a_s|x_s)}{P_s(a_s \mid x_s)} \ell_s(a_s)$.
In this notation,
$
\frac{1}{t} \sum_{s=1}^t \tilde{c}_s^h(\pi_s(x_s)) = \frac{1}{t} \sum_{s=1}^{t} X_s
$.
Observe that
\begin{align}
\EE\sbr{X_s \mid \Bcal_{s-1}} &= \EE_{(x_s, \ell_s) \sim \Dcal} \EE_{a_s \sim P_s(\cdot \mid x_s)} \frac{\pi_{s,h}(a_s \mid x_s)}{P_s(a_s \mid x_s)} \ell_s(a_s) \nonumber \\
&= \EE_{(x_s, \ell_s) \sim \Dcal} \EE_{a_s \sim \pi_{s,h}(\cdot \mid x_s)} \ell_s(a_s) = V(\pi_{s,h}).
\end{align}
Let $Z_s = X_s - \EE\sbr{X_s \mid \Bcal_{s-1}} = X_s - V(\pi_{s,h})$. It can be seen that $\cbr{Z_s}_{s=1}^t$ is a martingale difference sequence adapted to filtration $\cbr{\Bcal_s}_{s=0}^t$.

Let $M = \frac{1}{h \; p_{\min}}$; From the definition of $Z_s$, along with the facts that $P_s(a_s \mid x_s) \geq p_{\min}$, and $\pi_{s,h}(a_s \mid x_s) \in [0, \frac 1 h]$ with probability 1, we get that $\abr{Z_s(\pi)}\leq M$ with probability 1.

We now show an upper bound on the conditional variance of $Z_s$:
\begin{align*}
\EE \sbr{Z_s^2 \mid \Bcal_{s-1}}
&\leq \E \sbr{ X_s^2  \mid \Bcal_{s-1}} \\
&= \E_{(x_s,\ell_s) \sim D} \E_{a_s \sim P_s(\cdot \mid x_s)}  \sbr{\frac{\pi_{s, h}(a_s|x_s)^2}{P_s(a_s \mid x_s)^2} \ell_s(a_s)^2} \\
&\leq \E_{(x_s,\ell_s) \sim D} \E_{a_s \sim P_s(\cdot \mid x_s)}  \sbr{\frac{\pi_{s, h}(a_s|x_s)^2}{P_s(a_s \mid x_s)^2} \ell_s(a_s)} \\
&= \E_{(x_s,\ell_s) \sim D} \sbr{ \int_{[0,1]} \frac{\pi_{s, h}(a \mid x_s)^2}{P_s(a \mid x_s)^2} P_s(a \mid x_s) \ell_s(a) \diff a} \\
&= \E_{(x_s,\ell_s) \sim D} \int_{[0,1]} \frac{\pi_{s, h}(a \mid x_s)}{P_s(a \mid x_s)} \pi_{s, h}(a|x_s) \ell_s(a) \diff a \\
&\leq \E_{(x_s,\ell_s) \sim D} \frac{1}{p_{\min} h } \cdot \int_{[0,1]} \pi_{s, h}(a|x_s) \ell_s(a) \diff a = \frac{V(\pi_{s, h})}{p_{\min} h}.
\end{align*}
where the first inequality uses the fact that $\ell_s(a_s) \in [0,1]$, and the second inequality uses the facts that $\pi_h(a \mid x_s) \in [0, \frac1h]$, and $P_s(a_s \mid x_s) \geq p_{\min}$. Consequently, $\sum_{s=1}^t \EE \sbr{Z_s^2 \mid \Bcal_{s-1}} \leq \frac{1}{p_{\min} h} \sum_{s=1}^t V(\pi_{s, h})$.

Applying Lemma~\ref{lem:freedman} on $Z_s$'s, with $n = t$, $M = \frac{1}{h p_{\min}}$, $\delta = \delta'$, we have that with probability $1-\delta'$:
\[
\abr{\sum_{s=1}^t \tilde{c}_s^h(\pi_s(x_s)) - \sum_{s=1}^t V(\pi_{s,h})} \leq 4 \sqrt{ \rbr{\sum_{s=1}^t V(\pi_{s,h})} \frac{\ln\frac{2 t}{\delta'}}{p_{\min}h}} + \frac{2 \ln\frac{2 t}{\delta'}}{p_{\min}h}.
\]


The first item now follows from dividing both sides of the above inequality by $t$.

We now use the first item to show the second item. Fix a $\pi$ in $\Pi$. We take $\cbr{\pi_s}_{s=1}^t$ such that $\pi_s = \pi$ for all $s$. By the previous item, we have that with probability $1-\frac{\delta'}{\abr{\Pi}}$,
\begin{align*}
\abr{ \frac 1 t \sum_{s=1}^t \tilde{c}_s^h(\pi(x_s)) - \frac 1 t \sum_{s=1}^t V(\pi_{h}) }
& \leq
4 \sqrt{ \rbr{\frac 1 t \sum_{s=1}^t V(\pi_{h})} \frac{\ln\frac{2 \abr{\Pi} t}{\delta'}}{t \; p_{\min} \; h}} + \frac{2 \ln\frac{2 \abr{\Pi} t}{\delta'}}{t \; p_{\min} \; h} \\
& \leq
4 \sqrt{ \frac{\ln\frac{2 \abr{\Pi} t}{\delta'}}{t \; p_{\min} \; h}} + \frac{2 \ln\frac{2 \abr{\Pi} t}{\delta'}}{t \; p_{\min} \; h}.
\label{eqn:pv-conv2}
\end{align*}
We conclude the item by taking a union bound on all $\pi$ in $\Pi$. 
\end{proof}

\section{Experimental Details}
\label{sec:experimental_details}
Of the six datasets five were selected randomly from OpenML with the criterion of having millions of samples with unique regression values. These include  \texttt{wisconsin}, \texttt{cpu\_act}, \texttt{auto\_price},  \texttt{black\_friday} (customer purchases on black Friday) and \texttt{zurich\_delay} (Zurich public transport delay data). We also included a synthetic dataset, namely \texttt{ds}, which was created by linear regression of standard gaussians with additive noise. \\
Our main comparator is the discretized $\epsilon$-greedy algorithm $\dl$ in Vowpal Wabbit which by default uses the doubly robust approach~\cite{dudik2011doubly} for policy evaluation and optimization. This method reduces to cost-sensitive one-against-all multi-class classification which has computational complexity linear w.r.t number of discrete actions.
Our other comparator is $\dt$, the discretized filter tree which is equivalent to \cats without smoothing, i.e. with zero bandwidth. For all the approaches we used $\epsilon = 0.05$ and a parameter free update rule based on coin betting~\cite{Coin16}. \\
We implemented \cats in Vowpal Wabbit. The details of the implementation are explained in the next section.

\section{CATS implementation with $\order{\log K}$ time per example}
\label{sec:cats-implementation}

In this section, we present the details of our online implementation of \cats that has $\order{\log K}$ time cost per example. Our implementation can be generalized to the setting where the action space $\Acal$ is a continuous interval in $\RR$; for simplicity of presentation, we focus on $\Acal = [0,1]$ in this section. Before going into the details, we introduce some additional notation.

Recall that $K = 2^D$ is the discretization level; the corresponding discretized action space is defined as $\Acal_K = \cbr{0, \frac{1}{K}, \ldots, \frac{K-1}{K}}$. We will consider choices of bandwidth $h$ in $\Hcal_K = \cbr{2^{-i}: i \in [K]}$; our algorithm can be easily generalized to other values of $h$'s, by modifying the tree initialization procedure. For a bandwidth $h$ in $\Hcal_K$, define an auxiliary parameter $m^{\#} = \log_2(K \cdot h)$, which is an integer. It can be easily seen that $h = 2^{m^{\#}} / K$.


\begin{algorithm}
 \caption{\treetrain with no data partitioning}
 \begin{algorithmic}[1]
    \STATE Input: $K=2^D$, $\Fcal$, training data $\{(x_s,c_s)\}_{s=1}^n$ with $c_s \in \RR^K$
    \FOR{level $d$ from $D-1$ down to $0$}
    \FOR{nodes $\vt$ at level $d$}  
    \STATE For each $(x_s,c_s)$ define binary cost $c_s^{\vt}$ with
    \begin{align*}
        c_s^{\vt}(\leftt) &= c_s(\vt.\leftt.\getaction(x_s))\\
        c_s^{\vt}(\rightt) &= c_s(\vt.\rightt.\getaction(x_s)).
    \end{align*}
    \STATE Train $f^{\vt} \in \Fcal$ on $S^{\vt} = \{(x_s,c_s^{\vt}): s \in [n], c_s^{\vt}(\leftt) \ne c_s^{\vt}(\rightt)\}$:
    \begin{align*}
        f^{\vt} \in \argmin_{f \in \Fcal} \EE_{S^\vt} \sbr{c^\vt(f(x))}.
    \end{align*}
    \label{line:tree-node-erm2}
    \ENDFOR
    \ENDFOR
    \STATE Return tree $\Tcal$ with $\{f^{\vt}\}$ as node classifiers.
 \end{algorithmic}
 \label{alg:tree-train-no-part}
\end{algorithm}

\begin{algorithm}
  \begin{algorithmic}
    \REQUIRE{Tree $\tree$ with depth $D$, $m^{\#}$}
    \COMMENT{Initialize a tree policy $\tree$ with $K = 2^D$ leaves by assigning $\id$'s to each node; in addition, initialize the nodes such that the leftmost and rightmost $2^{m^{\#}}$ leaves are unreachable}
    \ENSURE{Initialized tree policy $\tree$ }
    \STATE $\tree.\roott.\id \gets 0$
    \FOR{level $d$ in $\cbr{0, \ldots, D-1}$}
    \FOR{nodes $\vt$ at level $d$ of $\tree$}
    \STATE Initialize the online CSMC base learner at $\vt$
    \STATE $\vt.\leftt.\id =2 \times \vt.\id+1$
    \STATE $\vt.\rightt.\id =2 \times \vt.\id+2$
    \ENDFOR
    \ENDFOR
   \FOR{nodes $\vt$ at level $D$ of $\tree$}
  	\STATE Set $\labelt(\vt) \gets ( \vt.\id - (2^{D}-1) ) / K$.
  \ENDFOR

    \STATE Set $\vt^\onlyright$ to be the node $\vt$ with $\vt.\id = 2^{D-{m^{\#}}-1} - 1$, and let $f^{\vt^\onlyright} \equiv \rightt$.
  \STATE Set $\vt^{\texttt{only-left}}$ to be the node $\vt$ with $\vt.\id = 2^{D-{m^{\#}}} - 2$, and let $f^{\vt^\onlyleft} \equiv \leftt$.

  \end{algorithmic}
  \caption{\buildtree}
  \label{alg:build-tree}
\end{algorithm}

\begin{algorithm}
  \begin{algorithmic}[1]
    \REQUIRE Tree policy \tree, context $x$, cost vector $\tilde{c}$ implicitly represented by
    actions $\mina, \maxa$ in $\Acal_K$ and cost $c^*$ in $\RR_+$, such that for all $a \in \Acal_K$, $\tilde{c}(a) = c^*$ if $a \in [\mina, \maxa]$, and $\tilde{c}(a) = 0$ otherwise.
    \ENSURE Updated tree policy \tree.

     \STATE $\at \gets$ leaf corresponding to action $\mina$, $\bt \leftarrow$ leaf corresponding to action $\maxa$
     \STATE $\at.\cost \gets c^*$, $\bt.\cost \gets c^*$
     \label{line:cost-leaf}
     \STATE $\alpha_D \gets \at$, $\beta_D \gets \bt$.
\FOR{level $d$ from $D$ down to $1$}
    \IF{$\alpha_d.\parent \; \neq \beta_d.\parent$}
    	\STATE $S_d \gets \cbr{ \alpha_d, \beta_d }$
    \ELSE
    	\STATE $S_d \gets \cbr{\alpha_d}$;
    \ENDIF
	\FOR{nodes $\vt \in \texttt{$S_d$}$}
	\STATE $\ut \gets \vt.\parent$ \COMMENT{Goal: update the online learner in $\ut$, the parent of $\vt$}
	
	\IF{$\ut \in \cbr{ \vt^{\onlyleft}, \vt^{\onlyright } }$}
	    \STATE \textbf{continue}; \COMMENT{No updates on $\vt^{\onlyleft}$ and $\vt^{\onlyright}$}
	    \label{line:noupdate}
	\ENDIF
	
	\STATE $\wt \gets$ the sibling of node $\vt$. \COMMENT{Create cost vector $c^\ut$}
	\STATE $\wt.\cost \gets \rc(\wt, \alpha_d, \beta_d)$
	\IF{ $\vt = \ut.\leftt$ }
		\STATE $c^\ut(\leftt) \gets \vt.\cost, c^\ut(\rightt) \gets \wt.\cost$ \COMMENT{$\vt$ is the left child of $\ut$}
		\label{line:cost-construction-1}
	\ELSE
		\STATE $c^\ut(\leftt) \gets \wt.\cost, c^\ut(\rightt) \gets \vt.\cost$ \COMMENT{$\vt$ is the right child of $\ut$}
		\label{line:cost-construction-2}
	\ENDIF
	
	\STATE $\ut.\learn(f^\ut, (x, c^\ut) )$ \COMMENT{Update the online CSMC base learner in $\ut$}
	\label{line:update-u}
	\STATE $\ut.\cost \gets c^\ut(f^\ut(x))$ \COMMENT{Compute $c(\tree^{\ut}(x))$ for training in nodes of higher level}
	\label{line:cost-u}

	\ENDFOR
 	\STATE $\alpha_{d-1} \gets \alpha_d.\parent$ \COMMENT{Compute the ancestors of $\at$, $\bt$ to a level up}
 	\STATE $\beta_{d-1} \gets \beta_d.\parent$
\ENDFOR

  \end{algorithmic}
  \caption{\onlinetreetrain with $\order{\log K}$ time cost per example} 
  \label{alg:tree-learn}
\end{algorithm}

\begin{algorithm}
  \caption{\rc}
  \label{alg:return-cost}
\begin{algorithmic}
\REQUIRE Tree node $\wt$; Tree nodes $\alpha$ and $\beta$, which are the ancestors of $\at$ and $\bt$, respectively, at the same level of $\wt$.
\IF{$\wt.\id < \alpha.\id$ or $\wt.\id > \beta.\id$}
\RETURN 0
\ELSIF{$\alpha.\id < \wt.\id < \beta.\id$}
\RETURN $c^*$
\ELSIF{$\wt.\id = \alpha.\id$}
\RETURN $\alpha.\cost$
\ELSIF{$\wt.\id = \beta.\id$}
\RETURN $\beta.\cost$
\ENDIF
\end{algorithmic}
\end{algorithm}

\subsection{\buildtree: initialization of tree policy}

We now describe a procedure \buildtree, namely Algorithm~\ref{alg:build-tree}, that provides essential initialization of our tree policy $\tree$. First, \buildtree assigns a unique $\id$ for each node $\vt$ in the tree $\tree$ through traversing the tree in a top-down fashion. It also supplies the action labels of all $K$ leaves.
The nodes' $\id$'s are assigned such that within the same level, the $\id$'s are increasing from left to right.
Furthermore, it initializes the online binary CSMC base learners in all its internal nodes. To ensure $\order{\log K}$ time cost of the tree learning algorithm, we disallow actions in $\Acal_K \cap [0, h]$ and $\Acal_K \cap [1 - h, 1]$ to be taken by the tree policy. To this end, two classifiers in nodes $\vt^{\onlyleft}$ and $\vt^{\onlyright}$ are set to fixed classifiers $f^{\vt^{\onlyleft}} \equiv \leftt$ and $f^{\vt^{\onlyright}} \equiv \rightt$, both of which are {\em read-only}.

To see why the above restriction helps with ensuring $\order{\log K}$ time cost per example, we now recall the definition of the IPS CSMC example $(x_t, \tilde{c}_t)$ generated by log data $(x_t, a_t, \ell_t(a_t), P_t(a_t \mid x_t))$ in \cats. We first show that $\tilde{c}_t$ has a simple structure: if $a$ is in $\Acal_K \cap [h, 1-h]$, we have a concise formula of $\tilde{c}_t(a)$:
\begin{align}
\tilde{c}_t(a) &=
\begin{cases}
 		\frac{\ell_t(a_t)}{2 h P_t(a_t|x_t)}, & \abs{a - a_t} \leq h, \\
		0, & \text{otherwise}.
\end{cases}
\end{align}
Observe that $\tilde{c}_t$ is a piecewise constant function over $\Acal_K \cap [h, 1-h]$ with at most 3 pieces:
$[0,a_t-h]$ (if $a_t > h$), $[\max(0,a_t - h), \min(1,a_t + h)]$, and $[a_t + h, 1]$ (if $a_t < 1-h$).
The IPS cost vector $\tilde{c}_t$ can be summarized by three numbers: $c^* = \frac{\ell_t(a_t)}{2 h P_t(a_t|x_t)}$, the nonzero value in $\tilde{c}_t$, $\mina = \max(0, \frac{\lceil K(a_t - h) \rceil}{K})$, the minimum $a \in \Acal_K$ such that $\tilde{c}_t(a) = c^*$; $\maxa = \min(\frac{K-1}{K}, \frac{\lfloor K(a_t + h) \rfloor}{K})$, the maximum $a \in \Acal_K$ such that $\tilde{c}_t(a) = c^*$.

We remark that $\tilde{c}_t$ may not be a piecewise constant function {\em globally} over $\Acal_K$. This is because in general, $\tilde{c}_t(a) = \frac{\ell_t(a_t) \smooth(a_t \mid a)}{P_t(a_t|x_t)}  =  \frac{\ell_t(a_t) \one \rbr{a - a_t \leq h} }{ \vol( [a-h, a+h] \cap[0,1] ) \cdot P_t(a_t|x_t) }$, where $\vol(\cdot)$ denotes the Lebesgue measure. Therefore, if, say $a_t$ is in $[0, h]$, the induced IPS cost function $\tilde{c}_t$ can take many possible positive values for $a$ in region $[0,h]$, depending on the value of $\vol( [a-h, a+h] \cap[0,1])$. It turns out that enforcing the piecewise constant structure of the cost vector (as is done by restricting the CSMC vectors to only consider entries in $a$ in $\Acal_K \cap [h, 1-h]$) is vital to achieve $\order{\log K}$ per-example time cost, as we will see next.




\subsection{\onlinetreetrain: online update of tree policy}
\label{subsec:online_update}
In our implementation, to maximize data-efficiency, we will implement a more practical variant of \treetrain, namely Algorithm~\ref{alg:tree-train-no-part}; the difference between it and \treetrain is that, instead of partitioning the input data to train each level separately, we use the full input data to train nodes at all levels.

The tree policy training algorithm, namely \onlinetreetrain (Algorithm~\ref{alg:tree-learn}), is an online implementation of Algorithm~\ref{alg:tree-train-no-part}. It is used by \cats (in its line~\ref{line:tree-train}) to process the IPS CSMC example generated at every round $t$, to obtain an updated tree policy.
It receives a IPS CSMC example $(x, \tilde{c})$ as input, represented by context $x$, and $\mina$, $\maxa$, $c^*$ (representing $\tilde{c}$, as discussed in the previous section), and a tree \tree trained over previous CSMC examples $S$; specifically, $(x_t,\tilde{c}_t)$'s in \cats are its valid inputs.
Here we assume that the input $\tree$ is such that for every node $\vt$, its stored classifier $f^\vt$ is an approximation of $\argmin_{f \in \Fcal} \EE_{(x,c) \sim S} [c^\vt(f(x))]$ (recall the definition of $c^\vt$ in Algorithm~\ref{alg:tree-train-no-part}).
\onlinetreetrain updates the input \tree with $(x, \tilde{c})$, such that it approximates the output of \treetrain over $S \cup \cbr{(x,\tilde{c})}$, that is, for every node $\vt$, its stored classifier $f^\vt$ is an approximation of $\argmin_{f \in \Fcal} \EE_{(x,c) \sim S \cup \cbr{(x, \tilde{c})}} [c^\vt(f(x))]$.
Our online implementation replaces line~\ref{line:tree-train} of \cats with $\tree \gets \onlinetreetrain(\tree, (x_t, \tilde{c}_t))$, with the goal of ensuring the updated $\tree$ after round $t$ closely approximates $\treetrain(K,\Fcal,\{(x_s,\tilde{c}_s)\}_{s=1}^t)$.


The tree policy update proceeds in a bottom-up fashion. Given two leaves of the tree $\at, \bt$ that correspond to actions $\mina, \maxa$, we use them as ``seeds'' to ``climb up'' the tree, reaching nodes that need updating. Specifically, for every level $d \in [D]$, we maintain $\at_d$ and $\bt_d$ that correspond to the ancestors of $\at$ and $\bt$, respectively, at that level.

As discussed in the main text, for a given node $\vt$, if $c^{\vt}(\leftt) = c^{\vt}(\rightt)$, there is no need to update the online CSMC learner at $\vt$, because $f^\vt_{t+1}$, the ERM at node $\vt$ at time $t+1$, will be equal to $f^\vt_{t}$.
From Lemma~\ref{lem:update-correctness} below, it turns out that it suffices to only update the CSMC online learners in $\alpha_d$'s and $\beta_d$'s at levels $d \in \cbr{0,\ldots,D-1}$. In addition, to update an internal node $\vt$, one needs to obtain $c^\vt(\leftt)$ and $c^\vt(\rightt)$, which corresponds to costs of the action routed by its left and right subtrees, i.e. $\tilde{c}(\tree^{\vt.\leftt}(x))$ and $\tilde{c}(\tree^{\vt.\rightt}(x))$.
To ensure computational efficiency, Algorithm~\ref{alg:tree-learn} calls a carefully-designed subprocedure, namely $\rc$ (Algorithm~\ref{alg:return-cost}), that given any node $\vt$ at level $d$,
returns the cost $\tilde{c}(\tree^{\vt}(x))$ in {\em constant time}, provided that $\at_d, \bt_d$, the ancestors of $\at,\bt$ at the level $d$, have been identified. We refer the reader to Claim~\ref{claim:cost-v} for a proof of correctness of $\rc$.
Upon receiving binary CSMC example $(x, c^\vt)$, the CSMC oracle at node $\vt$ gets updated using an incremental update rule (such as stochastic gradient descent) on $(x, c^\vt)$ at line~\ref{line:update-u} of \onlinetreetrain, which we assume takes $\order{1}$ time (where the $\order{\cdot}$ notation here is only with respect to the discretization level $K$). Specifically, our implementation of \cats in Vowpal Wabbit uses base CSMC learners that performs a reduction from classification to online least-squares regression to approximate ERM: at every node, its corresponding base learner learns to predict the cost of going to the left and right branch respectively, and the learned classifier takes the branch with lower predicted cost. Furthermore, we use a parameter-free gradient update rule~\cite{Coin16} to implement our online least square regression procedure. As a result, in our implementation, the time costs of each base learner's prediction and update are both $\order{d}$, where $d$ is dimension of the context space.

We finally remark that in line~\ref{line:noupdate} of Algorithm~\ref{alg:tree-learn}, we skip updates on nodes $\vt^\onlyleft$ and $\vt^\onlyright$, ensuring that the tree policy never outputs actions in $\Acal_K \cap [0, h]$ or $\Acal_K \cap [1 - h, 1]$.


\subsubsection{Proof of correctness of \onlinetreetrain}

We now prove that Algorithm~\ref{alg:tree-learn} does not miss updating nodes that needs updates, i.e. the nodes $\ut$ such that $c^\ut(\leftt) \neq c^\ut(\rightt)$; recall that $c^\ut(\leftt) = \tilde{c}(\tree^{\ut.\leftt}(x))$ and $c^\ut(\rightt) = \tilde{c}(\tree^{\ut.\rightt}(x))$.

\begin{lemma}
For every internal node $\ut$ in $\tree$, if $c^\ut(\leftt) \neq c^\ut(\rightt)$, then Algorithm~\ref{alg:tree-learn} updates $\ut$ with binary cost-sensitive example $(x, c^\ut)$. Consequently, \onlinetreetrain (Algorithm~\ref{alg:tree-learn}) faithfully implements \treetrain (Algorithm~\ref{alg:tree-train}) in an online fashion.
\label{lem:update-correctness}
\end{lemma}
\begin{proof}
With the notations defined in \onlinetreetrain (Algorithm~\ref{alg:tree-learn}), denote by $\at$ (resp. $\bt$) the leaf with action label $\mina$ (resp. $\maxa$).
It can be seen from the description of \onlinetreetrain that if node $\ut$ is an ancestor of $\at$ or $\bt$, the base CSMC learner in $\ut$ will get updated. We now show that if $c^\ut(\leftt) \neq c^\ut(\rightt)$, $\ut$ must be an ancestor of either $\at$ or $\bt$, which will let us conclude that all nodes $\ut$ with $c^\ut(\leftt) \neq c^\ut(\rightt)$ will be updated.

We will prove the above statement's contrapositive: if neither $\at$ nor $\bt$ is a child of $\ut$, then $c^\ut(\leftt) = c^\ut(\rightt)$. Indeed, suppose $\ut$ is at level $d$, and denote by $\alpha_d$ and $\beta_d$ the ancestors of $\at$, $\bt$ at level $d$ respectively. Then, it must be the case that $\ut \neq \alpha_d$ and $\ut \neq \beta_d$.
From the first two items of Claim~\ref{claim:cost-v} below, we have that $\tilde{c}(a)$ must agree unanimously for all  actions $a$ in $\range(\tree^\ut)$. Now, because both $c^\ut(\leftt)$ and $c^\ut(\rightt)$ take values in $\range(\tree^\ut)$, they must also be equal.

In addition, from the last item in Claim~\ref{claim:cost-v} below, along with the description of \onlinetreetrain's lines~\ref{line:cost-construction-1} and~\ref{line:cost-construction-2}, if node $\ut$ gets updated, the $\leftt$ (resp. $\rightt$) entry of the binary cost vector $c^\ut(\leftt)$ (resp. $c^\ut(\rightt)$) takes value as $\ut.\leftt.\cost$ (resp. $\ut.\rightt.\cost$), which is $\tilde{c}(\tree^{\ut.\leftt}(x))$ (resp. $\tilde{c}(\tree^{\ut.\rightt}(x))$). Therefore the binary CSMC example $\ut$ receives is indeed $(x, c^\ut)$. This completes the proof of the lemma.
\end{proof}

\begin{claim}
For every level $d \in [D]$, denote by $\at_d$ and $\bt_d$ the ancestor of $\at$, $\bt$ at level $d$ in $\tree$ respectively. Then, for node $\vt$ at level $d$:
\begin{enumerate}
\item If $\vt.\id < \at_d.\id$ or $\vt.\id > \bt_d.\id$, then for all $a \in \range(\tree^\vt)$, $\tilde{c}(a) = 0$.
\item If $\at_d.\id < \vt.\id < \bt_d.\id$, then for all $a \in \range(\tree^\vt)$, $\tilde{c}(a) = c^*$.
\item If $\vt.\cost$ is available, it must equal $\tilde{c}(\tree^\vt(x))$; in addition, $\rc(\vt, \at_d, \bt_d)$ returns $\tilde{c}(\tree^\vt(x))$ correctly.  
\end{enumerate}
\label{claim:cost-v}
\end{claim}
\begin{proof}
 It can be seen that for every node $\ut$ at level $d$, $\range(\tree^\ut)$ spans a separate contiguous subinterval of $[0, 1]$. Specifically, for every $\ut$ at level $d$, define interval
\[ I_\ut = \left[ \frac{\ut.\id - (2^d - 1) }{2^d}, \frac{\ut.\id + 1 - (2^d - 1) }{2^d} \right), \]
we have $\range(\tree^\ut) = \range(\tree) \cap I_\ut$, and all $I_\ut$'s are disjoint for $\ut$'s at level $d$.

For the first item, suppose $\vt.\id < \at_d.\id$, i.e. $\vt$ is to the left of $\at_d$. In this case, all elements of $\range(\tree^\vt)$ must be less than $\mina$, and therefore for all $a \in \range(\tree^\vt)$, $c^\vt(a) = 0$. A similar reasoning applies to the case when $\vt.\id > \bt_d.\id$.

For the second item, suppose $\at_d.\id < \vt.\id < \bt_d.\id$, i.e. $\vt$ is in the middle of $\at_d$ and $\bt_d$. In this case, all elements of $\range(\tree^\vt)$ must be within the interval $[\mina, \maxa]$, therefore, by the definition of $\mina$ and $\maxa$, we have that for all $a \in \range(\tree^\vt)$, $c^\vt(a) = c^*$.

For the last item, we consider two cases.
\begin{enumerate}
\item If $\vt \neq \at_d$ and $\vt \neq \bt_d$, then from the first two items we have just shown, we can decide the value of $c^\vt(\tree^\vt(x))$ directly by comparison with the $\id$'s of $\alpha$ and $\beta$, which is consistent with the implementation of $\rc$; also note that in this case, $\vt.\cost$  gets assigned to $\rc(\vt, \at_d, \bt_d)$, which also equals $c^\vt(\tree^\vt(x))$.

\item Otherwise, $\vt = \at_d$ or $\vt = \bt_d$. In this case, $\rc$ returns the stored cost of $\vt$, i.e. $\vt.\cost$.
It suffices to show that $\at_d.\cost$ (resp. $\bt_d.\cost$), is indeed $\tilde{c}(\tree^{\at_d}(x))$ (resp. $\tilde{c}(\tree^{\bt_d}(x))$), which we show by induction:

\paragraph{Base case.} In the case when $d = D$, $\at_D.\cost = \at.\cost$ (resp. $\beta_D.\cost = \bt.\cost$) is directly calculated in line~\ref{line:cost-leaf} of Algorithm~\ref{alg:tree-learn}, and is indeed $\tilde{c}(\labelt(\at)) =  c^\vt(\at)$ (resp. $\tilde{c}(\labelt(\bt)) = c^\vt(\bt)$), and is equal to $c^*$.

\paragraph{Inductive case.} Suppose for level $d+1$, $\rc(\ut, \alpha_{d+1}, \beta_{d+1})$ returns $c(\tree^{\ut}(x))$ correctly for $\ut$  in $\cbr{ \alpha_{d+1}, \beta_{d+1} }$.
Now consider a node $\vt$ at level $d$, which is either $\at_d$ or $\bt_d$. By inductive hypothesis, and the correctness of $\rc$ on the costs of non-ancestors of $\at$,$\bt$ in the last item, for both $\vt.\leftt$ and $\vt.\rightt$, their costs $c^\vt(\leftt) = \tilde{c}(\tree^{\vt.\leftt}(x))$ and $c^\vt(\rightt) = \tilde{c}(\tree^{\vt.\rightt}(x))$ are calculated correctly by $\rc$. Hence, the cost calculated by $\rc$ on node $\vt$, $\vt.\cost$, at line~\ref{line:cost-u} in \onlinetreetrain, equals $c^\vt(f^\vt(x)) = \tilde{c}( \tree^{\vt.f^\vt(x)}(x) ) = \tilde{c}(\tree^\vt(x))$.
This completes the induction.
\end{enumerate}
The proof of the last item is complete.
\end{proof}

\subsection{Proof of Theorem~\ref{thm:cats-log-time}}

We are now ready to prove the time complexity guarantee of \cats, i.e. Theorem~\ref{thm:cats-log-time} in the main body.

\begin{proof}[Proof of Theorem~\ref{thm:cats-log-time}]
From Lemma~\ref{lem:update-correctness}, we see that \onlinetreetrain faithfully implements \treetrain in an online fashion. As other steps of \cats are intact, the online implementation of \cats faithfully implements the original \cats.

Moreover, consider the operations of \cats at every time step:
\begin{enumerate}
\item Predict $\tree(x)$: this takes $\order{D} = \order{\log K}$ time as can be directly seen from Algorithm~\ref{alg:filter-tree-pred}.
\item Generate $\epsilon$-greedy action distribution, take action, create $(x_t, \tilde{c}_t)$ implicitly by representing $\tilde{c}_t$ as $(\mina, \maxa, c^*)$: these steps take $\order{1}$ time as they are based on manipulations of piecewise constant density with at most 3 pieces.
\item $\onlinetreetrain(\tree, (x_t, \tilde{c}_t) )$: this takes $\order{D} = \order{\log K}$ time, because at each of the $D$ levels, there are at most 2 nodes to be updated, and for every such node, $\rc$ takes $\order{1}$ time to retrieve the costs of both subtrees.
\end{enumerate}
In summary, the total time cost of \cats at every time step is $\order{\log K}$.
\end{proof}
\section{Additional Experimental Results}
\label{sec:add-experiments}
Additional figures comparing running times of \cats against \dl and \dt for the rest of the datasets  are shown in Figures~\ref{fig:time2}-\ref{fig:time6}.
\begin{figure}[htp!]
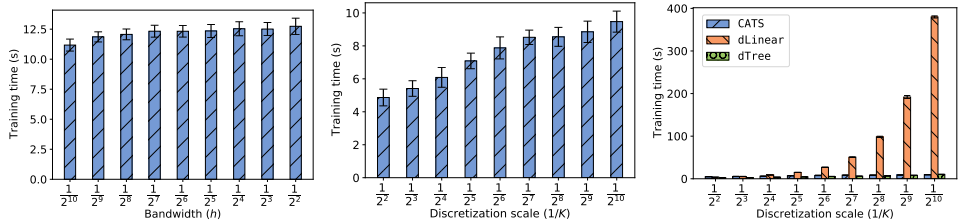

  \begin{center}
\includegraphics[width=0.3\textwidth]{Figures/BNG_cpu_act_timeh.png}
\includegraphics[width=0.3\textwidth]{Figures/BNG_cpu_act_timen.png}
\includegraphics[width=0.3\textwidth]{Figures/BNG_cpu_act_timenn.png}
\end{center}
\caption{Training time of \cats (blue bar) w.r.t: (\textbf{left}) bandwidth ($h$) with a fixed discretization scale $K = 2^{13}$; (\textbf{middle}) discretization scale ($1/K$) with a fixed $h = 1/4$; (\textbf{right}) discretization scale ($1/K$) with a fixed $h = 1/4$, compared against \dl (orange bar) and \dt (green bar), in the \texttt{cpu\_act} dataset.}
\label{fig:time2}
\end{figure}

\begin{figure}[htp!]
  \begin{center}
\includegraphics[width=0.3\textwidth]{Figures/zurich_timeh.png}
\includegraphics[width=0.3\textwidth]{Figures/zurich_timen.png}
\includegraphics[width=0.3\textwidth]{Figures/zurich_timenn.png}
\end{center}
\caption{Training time of \cats (blue bar) w.r.t: (\textbf{left}) bandwidth ($h$) with a fixed discretization scale $K = 2^{13}$; (\textbf{middle}) discretization scale ($1/K$) with a fixed $h = 1/4$; (\textbf{right}) discretization scale ($1/K$) with a fixed $h = 1/4$, compared against \dl (orange bar) and \dt (green bar), in the \texttt{zurich\_delay} dataset.}
\label{fig:time3}
\end{figure}

\begin{figure}[htp!]
  \begin{center}
\includegraphics[width=0.3\textwidth]{Figures/BNG_wisconsin_timeh.png}
\includegraphics[width=0.3\textwidth]{Figures/BNG_wisconsin_timen.png}
\includegraphics[width=0.3\textwidth]{Figures/BNG_wisconsin_timenn.png}
\end{center}
\caption{Training time of \cats (blue bar) w.r.t: (\textbf{left}) bandwidth ($h$) with a fixed discretization scale $K = 2^{13}$; (\textbf{middle}) discretization scale ($1/K$) with a fixed $h = 1/4$; (\textbf{right}) discretization scale ($1/K$) with a fixed $h = 1/4$, compared against \dl (orange bar) and \dt (green bar), in the \texttt{wisconsin} dataset.}
\label{fig:time4}
\end{figure}

\begin{figure}
  \begin{center}
\includegraphics[width=0.3\textwidth]{Figures/black_friday_timeh.png}
\includegraphics[width=0.3\textwidth]{Figures/black_friday_timen.png}
\includegraphics[width=0.3\textwidth]{Figures/black_friday_timenn.png}
\end{center}
\caption{Training time of \cats (blue bar) w.r.t: (\textbf{left}) bandwidth ($h$) with a fixed discretization scale $K = 2^{13}$; (\textbf{middle}) discretization scale ($1/K$) with a fixed $h = 1/4$; (\textbf{right}) discretization scale ($1/K$) with a fixed $h = 1/4$, compared against \dl (orange bar) and \dt (green bar), in the \texttt{black\_friday} dataset.}
\label{fig:time5}
\end{figure}
\begin{figure}[th!]
  \begin{center}
\includegraphics[width=0.3\textwidth]{Figures/BNG_auto_price_timeh.png}
\includegraphics[width=0.3\textwidth]{Figures/BNG_auto_price_timen.png}
\includegraphics[width=0.3\textwidth]{Figures/BNG_auto_price_timenn.png}
\end{center}
\caption{Training time of \cats (blue bar) w.r.t: (\textbf{left}) bandwidth ($h$) with a fixed discretization scale $K = 2^{13}$; (\textbf{middle}) discretization scale ($1/K$) with a fixed $h = 1/4$; (\textbf{right}) discretization scale ($1/K$) with a fixed $h = 1/4$, compared against \dl (orange bar) and \dt (green bar), in the \texttt{auto\_price} dataset.}
\label{fig:time6}
\end{figure}

\end{document}